\documentclass[a4paper]{article}
\usepackage[a4paper,top=3.5cm,bottom=3cm,left=3.5cm,right=3.5cm,marginparwidth=1.5cm]{geometry}
\usepackage{amstext}
\usepackage{amssymb}
\usepackage{amsmath}
\usepackage{amsthm}
\usepackage{nicefrac}
\usepackage{microtype}
\usepackage{xcolor}  
\usepackage{mathtools}

\usepackage[utf8]{inputenc}
\usepackage[T1]{fontenc}
\usepackage{booktabs}  
\usepackage{bbm}
\usepackage{times}
\usepackage{color}
\usepackage{array,float}
\usepackage{url}

\usepackage{paralist}
\usepackage{hyperref}

\usepackage{xparse,etoolbox}
\usepackage{threeparttable}
\usepackage{enumitem}
\usepackage{dsfont}

\usepackage[mathscr]{euscript}

\usepackage{lipsum,wrapfig}
\usepackage{multicol}
\usepackage[font=small,labelfont=bf]{caption}
\usepackage{verbatim}

\usepackage{latexsym}

\usepackage{multirow}
\usepackage{thmtools}
\usepackage{thm-restate}

\usepackage{MnSymbol}
\DeclareMathAlphabet\mathbb{U}{msb}{m}{n}
\usepackage{xpatch}

\usepackage{subfigure}

\usepackage{graphicx}
\usepackage{color}
\usepackage{graphics}

\usepackage{algorithmic}
\usepackage{algorithm}

\newif\ifnotes
\notestrue

\ifnotes

\newcommand{\rnote}[1]{ [\textcolor{magenta}{Raef: #1}] }
\newcommand{\mnote}[1]{ [\textcolor{red}{Mehryar: #1}] }
\newcommand{\thnote}[1]{ [\textcolor{purple}{Theertha: #1}] }

\else

\newcommand{\rnote}[1]{}
\newcommand{\mnote}[1]{}
\newcommand{\thnote}[1]{}
\fi

\def\Mset{\mathbb{M}}
\def\Rset{\mathbb{R}}
\def\Sset{\mathbb{S}}

\let\Pr\undefined

\DeclareMathOperator*{\Pr}{\mathbb{P}}

\DeclareMathOperator*{\E}{\mathbb E}

\DeclareMathOperator*{\argmin}{argmin}

\DeclareMathOperator{\diag}{\mathrm{diag}}
\DeclareMathOperator{\supp}{supp}
\DeclareMathOperator{\Tr}{Tr}
\DeclareMathOperator{\rank}{rank}
\DeclareMathOperator{\gap}{Gap}

\DeclarePairedDelimiter{\abs}{\lvert}{\rvert} 
\DeclarePairedDelimiter{\bracket}{[}{]}
\DeclarePairedDelimiter{\curl}{\{}{\}}
\DeclarePairedDelimiter{\norm}{\|}{\|}
\DeclarePairedDelimiter{\paren}{(}{)}
\DeclarePairedDelimiter{\tri}{\langle}{\rangle}

\newtheorem{proposition}{Proposition}

\newtheorem{lemma}{Lemma}
\newtheorem*{remark*}{Remark}
\newtheorem{defn}{Definition}

\newcommand{\hnabla}{\widehat{\nabla}}

\newcommand{\cA}{\mathcal{A}}
\newcommand{\cC}{\mathcal{C}}

\newcommand{\cL}{\mathcal{L}}

\newcommand{\cN}{\mathcal{N}}
\newcommand{\cO}{\mathcal{O}}

\newcommand{\cR}{\mathcal{R}}

\newcommand{\cQ}{\mathcal{Q}}

\newcommand{\cV}{\mathcal{V}}
\newcommand{\cW}{\mathcal{W}}
\newcommand{\cX}{\mathcal{X}}
\newcommand{\cY}{\mathcal{Y}}
\newcommand{\cZ}{\mathcal{Z}}

\newcommand{\sE}{{\mathscr E}}

\newcommand{\sH}{{\mathscr H}}

\newcommand{\sP}{{\mathscr P}}
\newcommand{\sQ}{{\mathscr Q}}

\newcommand{\sX}{{\mathscr X}}
\newcommand{\sY}{{\mathscr Y}}

\newcommand{\bI}{{\mathbf I}}

\newcommand{\bM}{{\mathbf M}}

\newcommand{\bU}{{\mathbf U}}

\newcommand{\bX}{{\mathbf X}}

\newcommand{\be}{{\mathbf e}}
\newcommand{\bg}{{\mathbf g}}

\newcommand{\bzero}{{\mathbf 0}}

\newcommand{\sfq}{{\mathsf q}}

\newcommand{\dis}{\mathrm{dis}}
\newcommand{\Dis}{\mathrm{Dis}}

\newcommand{\Rad}{\mathcal{R}}

\newcommand{\h}{\widehat}
\newcommand{\ov}{\overline}
\newcommand{\wt}{\widetilde}
\newcommand{\e}{\epsilon}

\newcommand{\Freg}{\wt F_T^{\lambda}}
\newcommand{\Fregprime}{\wt F_{T'}^{\lambda}}
\newcommand{\tx}{\widetilde{x}}
\DeclareMathOperator{\B}{\mathbb{B}} 
\newcommand{\lap}{\mathsf{Lap}}
\usepackage[algo2e,ruled]{algorithm2e}

\newcommand{\ignore}[1]{}

\hypersetup{
  colorlinks   = true,
  urlcolor     = blue,
  linkcolor    = blue,
  citecolor    = blue
}

\title{Private Domain Adaptation from a Public Source}
\author{%
  Raef Bassily\thanks{The Ohio State University \& Google Research NY. ~
  \texttt{bassily.1@osu.edu, raef@google.com}}
  \and
  Mehryar Mohri\thanks{Google Research \& Courant Institute of Mathematical Sciences, NY. ~
  \texttt{mohri@google.com}}
  \and
  Ananda Theertha Suresh\thanks{Google Research, NY. ~
  \texttt{theertha@google.com}}
}
\date{}

\usepackage[toc, page, header]{appendix}
\begin{document}

\maketitle

\begin{abstract}

A key problem in a variety of applications is that of domain
adaptation from a public source domain, for which a relatively large
amount of labeled data with no privacy constraints is at one's
disposal, to a private target domain, for which a private sample is
available with very few or no labeled data. In regression problems
with no privacy constraints on the source or target data, a
\emph{discrepancy minimization} algorithm based on several theoretical
guarantees was shown to outperform a number of other adaptation
algorithm baselines. Building on that approach, we design
differentially private discrepancy-based algorithms for adaptation
from a source domain with public labeled data to a target domain with
unlabeled private data. The design and analysis of our private
algorithms critically hinge upon several key properties we prove for a
smooth approximation of the weighted discrepancy, such as its
smoothness with respect to the $\ell_1$-norm and the sensitivity of
its gradient. Our solutions are based on private variants of
Frank-Wolfe and Mirror-Descent algorithms. We show that our adaptation
algorithms benefit from strong generalization and privacy guarantees
and report the results of experiments demonstrating their
effectiveness.

\end{abstract}

\section{Introduction}

In a variety of applications in practice, the amount of labeled data
available from the domain of interest is too modest to train an
accurate model. Instead, the learner must resort to using labeled
samples from an alternative source domain, whose distribution is
expected to be close to that of the target domain. Additionally,
typically a large amount of \emph{unlabeled data} from the target
domain is also at one's disposal.

The problem of generalizing from that distinct source domain to a
target domain for which few or no labeled data is available is a
fundamental challenge in learning theory and algorithmic design known
as the \emph{domain adaptation problem}.
We study a privacy-constrained and thus even more demanding scenario
of domain adaptation, motivated by the critical data restrictions in
modern applications: in practice, often the labeled data available
from the source domain is public with no privacy constraints, but the
unlabeled data from the target domain is subject to privacy
constraints. 

Differential privacy has become the gold standard of privacy-preserving data analysis as it offers formal and quantitative privacy guarantees and enjoys many attractive properties from an algorithmic design perspective \cite{DR14}. Despite the remarkable progress in the field of differentially private machine learning, the problem of differentially private domain adaptation is still not well-understood. In this work, we present several new differentially private adaptation algorithms for
the scenario described above that we show benefit from strong generalization guarantees. We also report the results of experiments
demonstrating their effectiveness. Note that there has been a sequence of publications that provide formal differentially private learning guarantees assuming access to public data \cite{chaudhuri2011sample,beimel2013private, bassily2018model,ABM19,nandi2020privately,bassily2020private}. However, their results are not applicable 
to the adaptation problem we study, since they assume that the source 
and target domains coincide.

The design of our algorithms and their guarantees benefit from the
theoretical analysis of domain adaptation by a series of prior
publications, starting with the introduction of a
\emph{$d_A$-distance} between distributions by
\cite{KiferBenDavidGehrke2004} and
\cite{BenDavidBlitzerCrammerPereira2006}. These authors used this
notion to derive learning bounds for the zero-one loss, (see also the
follow-up publications \cite{BlitzerCrammerKuleszaPereiraWortman2008,
  BenDavidBlitzerCrammerKuleszaPereiraVaughan2010}) in terms of a
quantity denoted by $\lambda_{\sH}$ that depends on the hypothesis set
$\sH$ and the distribution and that cannot be estimated from
observations.
Later, \cite*{MansourMohriRostamizadeh2009} and
\cite{CortesMohri2014} introduced the notion of \emph{discrepancy},
which they used to give a general analysis of single-source adaptation
for arbitrary loss functions. The notion of discrepancy is a
divergence measure tailored to domain adaptation that coincides with
the $d_A$-distance in the special case of the zero-one loss. Unlike
other divergence measures between distributions such as an
$\ell_1$-distance, discrepancy takes into account the loss function
and the hypothesis set and, crucially, can be estimated from finite
samples.  The authors 
presented Rademacher complexity learning bounds
in terms of the discrepancy for arbitrary hypothesis sets and loss
functions, as well as pointwise learning bounds for kernel-based
hypothesis sets.

For regression problems with no privacy constraints on the source or
target data, \cite{CortesMohri2014} gave a \emph{discrepancy
minimization algorithm} based on a reweighting of the losses of sample
points. They further presented a series of experimental results
demonstrating that their algorithm outperformed all other baselines in
a series of tasks.
Building on that approach, we design new differentially private
discrepancy-based algorithms for adaptation from a source domain with
public labeled data to a target domain with unlabeled private data. In
Section~3, we briefly present some background material on the
discrepancy analysis of adaptation motivating that approach.

The design and analysis of our private algorithms crucially hinge upon
several key properties we prove for a smooth approximation of the
weighted discrepancy, such as its smoothness with respect to the
$\ell_1$-norm and the sensitivity of its gradient (Section~4).  In
Section~5, we present new two-stage adaptation algorithms that can be
viewed as private counterparts of the discrepancy minimization
algorithm of \cite{CortesMohri2014}. As with that algorithm, the
first stage consists of finding a reweighting of the source sample
that minimizes the discrepancy, the second stage of minimizing a
regularized weighted empirical loss based on the reweighting found in
the first stage. Since the second stage does not involve private data,
only the first stage requires a private solution.  Our solutions are 
based on private variants of Frank-Wolfe and Mirror-Descent
algorithms, and they are computationally efficient. We describe these solutions in detail and prove privacy and
generalization guarantees for both algorithms. 
We further compare the
benefits of these algorithms as a function of the sample sizes. 

In Section~6, we present a new, computationally efficient, \emph{single-stage} differentially private adaptation
algorithm seeking to directly minimize the sum of the weighted empirical loss and the discrepancy. Since attaining the minimum in this case is generally intractable due to non-convexity of the objective, instead, our algorithm finds an approximate stationary point of this objective. Our algorithm is comprised of a sequence of Frank-Wolfe updates, where each update consists of a differentially private update of the weights and a non-private update of the predictor. In fact, our algorithm can be used in much more general settings of private non-convex optimization over (a Cartesian product of) domains with different geometries. We formally prove the privacy and convergence guarantees of our algorithm in a general problem setting, and then derive its generalization guarantees in the context of adaptation. 
%
Finally, in Section~7, we report our experimental results. 

We start
with the introduction of preliminary concepts and definitions relevant
to our analysis.

\section{Preliminaries}

Let $\sX \subset \Rset^d$ denote the input space and $\sY$ the
output space, which we assume to be a measurable subset of
$\Rset$. We
assume that $\sX$ is included in the $\ell_2$ ball of
radius $r$, $\B^d(r)$. We will also assume that $\sY$ is included in a bounded interval of diameter $Y>0$. Let $\sH$ be a family of hypotheses mapping from $\sX$ to
$\sY$. We focus on the family of linear hypotheses $\sH = \curl*{x
  \mapsto w \cdot x \colon \norm{w} \leq \Lambda}$. We will be mainly interested in the regression setting, though
some of our results can be extended to other contexts. For any $h \in
\sH$, we denote by $\ell(h(x), y) = (h(x) - y)^2$ the familiar squared
loss of $h$ for the labeled point $(x, y) \in \sX \times \sY$. We
denote by $M > 0$ an upper bound on the loss: $\ell(h(x), y) \leq M$,
for all $(x, y)$. 

\paragraph{Learning scenario:} We identify a domain with a distribution
over $\sX \times \sY$ and refer to the source domain as the one
corresponding to a distribution $\sQ$ and the target domain, the one
corresponding to a distribution $\sP$.
We assume that the learner receives a sample $S = \paren*{(x_1, y_1),
  \ldots, (x_m, y_m)}$ of $m$ labeled points drawn i.i.d.\ from a
distribution $\sQ$ over $\sX \times \sY$ and it also has access to a large sample $T$ of $n$ unlabeled points drawn i.i.d.\ from
$\sP_\sX$, the input marginal distribution associated to $\sP$. We
 view the data from $\sQ$, that is sample $S$, as \emph{public
data}, and the data from $\sP$, sample $T$, as \emph{private data}.

The objective of the learner is to use the samples $S$ and $T$ to 
select a hypothesis $h \in \sH$ with small expected loss with 
respect to the target domain:
 $\cL(\sP, h) = \E_{(x, y) \sim \sP}[\ell(h(x), y)]$.
In the absence of any constraints, this coincides with the standard
problem of single-source domain adaptation, studied in a very broad
recent literature, starting with the theoretical studies of
\cite{BenDavidBlitzerCrammerPereira2006,MansourMohriRostamizadeh2009,
  CortesMohri2014}.

\paragraph{Discrepancy notions:} Clearly, the success of adaptation
depends on the closeness of the distributions $\sP$ and $\sQ$, which
can be measured according to various divergences.  The notion of
\emph{discrepancy} has been shown to be appropriate measure of
divergence between distributions in the context of domain
adaptation. We will distinguish the so-called \emph{$\sY$-discrepancy}
$\dis_\sY(\sP, \sQ)$, which can only be estimated when sufficient
labeled data is available from both distributions, and the standard
discrepancy $\dis(\sP, \sQ)$, which can be estimated from finite
unlabeled samples from both distributions:
\begin{align*}
\Dis_\sY(\sP, \sQ) 
& = \max_{h \in \sH} \curl*{ \cL(\sP, h) - \cL(\sQ, h)}\\
\Dis(\sP, \sQ)
& =
\max_{h, h' \in \sH} \bracket*{
\E_{x \sim \sP_\sX}
[
\ell(h(x), h'(x))]
- 
\E_{x \sim \sQ_\sX}
[\ell(h(x), h'(x))]
}.
\end{align*}
We will be using the two-sided versions
of these expressions. For example, we will use 
$$\dis(\sP, \sQ)\triangleq \max\curl*{\Dis(\sP, \sQ), \Dis(\sQ, \sP) }$$ though part of
our analysis holds with one-sided definitions too.

\paragraph{Matrix definitions:} We will adopt the following matrix
definitions and notation.  We denote by $\Mset_d$ the set of
real-valued $d \times d$ matrices and by $\Sset_d$ the subset of
$\Mset_d$ formed by symmetric matrices.  We will denote by
$\tri{\cdot, \cdot}$ the Frobenius product defined for all $\bM, \bM'
\in \Mset_d$ by $\tri{\bM, \bM'} = \sum_{i, j = 1}^d \bM_{ij}
\bM'_{ij} = \Tr(\bM^\top \bM')$.
For any matrix $\bM \in \Sset_d$, we denote by $\lambda_i(\bM)$ the
$i$th eigenvalue of $\bM$ in decreasing order and will also denote by
$\lambda_{\max}(\bM) = \lambda_{1}(\bM)$ its largest eigenvalue, and
by $\lambda_{\min}(\bM) = \lambda_{d}(\bM)$ its smallest
eigenvalue. We also denote by $\lambda(\bM)$ the vector of eigenvalues
of $\bM$.
For any $p \in [1, +\infty]$, we will denote by $\norm{\bM}_{(p)}$
the $p$-Schatten norm of $\bM$ defined by
$\norm{\bM}_{(p)}
= \norm{\lambda(\bM)}_p
= \bracket*{\sum_{i = 1}^d \abs{\lambda_i(\bM)}^p}^{\frac{1}{p}}$.
Note that $p = +\infty$ corresponds to the spectral norm:
$\norm{\bM}_{(\infty)} = \norm{\lambda(\bM)}_{\infty}$,
which we also denote by $\norm{\bM}_2$.

\paragraph{Smoothness:} We will say that a continuously differentiable
function $f$ defined over a vector space $\sE$ is $\gamma$-smooth for
norm $\norm{\cdot}$ if
$\forall x, x' \in \sE, 
\norm{\nabla f(x) - \nabla f(x')}_* \leq \gamma \norm{x - x'}$,
where $\norm{\cdot}_*$ is the dual norm associated to
$\norm{\cdot}$. When $f$ is twice differentiable, it
is known that the condition on the Hessian
$\forall x, z \in \sE,  \abs*{z^\top \nabla^2 f(x) z} \leq \gamma \norm{z}^2$,
implies that $f$ is $\norm{\cdot}$-$\gamma$-smooth
\cite{Sidford2019}[Chapter 5; lemma 8].

\paragraph{Differential Privacy
  \cite{dwork2006calibrating,dwork2006our}:}  Let $\varepsilon,\delta
>0$. A (randomized) algorithm $\cA \colon \cZ^n \rightarrow \cR$ is
$(\varepsilon,\delta)$-differentially private if for all pairs of
datasets $T, T' \in \cZ$ that differ in exactly one entry, and every
measurable $\cO \subseteq \cR$, we have: $\Pr \left(\cA(T) \in \cO
\right) \leq e^\varepsilon \cdot \Pr \left(\cA(T') \in \cO \right)
+\delta.$
We consider differentially private algorithms that have access to an auxiliary public dataset $S$ in addition to their input private dataset $T$. In such case, we view the public set $S$ as being ``hardwired'' to the algorithm, and the constraint of differential privacy is imposed only w.r.t. the private dataset.

\section{Background on discrepancy-based generalization bounds}
\label{sec:background}

In this section, we briefly present some background material
on discrepancy-based generalization guarantees. A more
detailed discussion is presented in Appendix~\ref{app:background}.
Let 
the \emph{output label-discrepancy} $\eta_\sH(S, \wt T)$ be defined as
follows:
\[
\eta_\sH(S, \wt T) = \min_{h_0 \in \sH} \curl*{ 
\sup_{(x, y) \in S} |y - h_0(x)|
+ \sup_{(x, y) \in \wt T} |y - h_0(x)|
},
\]
where $\wt T$ is the labeled version of $T$ (i.e., $\wt T$ is $T$ associated with its true (hidden) labels). Note that $\dis(\sP_\sX, \sfq)$ measures the difference of the
distributions on the input domain.
In contrast, $\eta_\sH(S, \wt T)$ accounts for the difference of the
output labels in $S$ and $T$. 
Note that under the covariate-shift and separability assumption, we have
$\eta_\sH(S, \wt T) = 0$. In general, adaptation is not possible when
$\eta_\sH(S, \wt T)$ is large since the labels received on the training
sample would then be very different from the target ones.
Thus, we will assume, as in previous work, that we have 
$\eta_\sH(S, \wt T)\ll 1$. 
Then, the following learning bound, expressed in terms of the 
empirical unlabeled discrepancy $\dis \paren*{\h \sP_\sX, \sfq}$,
$\eta_\sH(S, \wt T)$, and 
the Rademacher complexity of $\sH$,
holds with probability at least
$1 - \beta$ for all $h \in \sH$ and all distributions $\sfq$
over $S$ \cite{CortesMohri2014,CortesMohriMunozMedina2019}:
\begin{equation}
\label{eq:disbound}
\cL(\sP, h)
\leq \sum_{i = 1}^m \sfq_i \ell(h(x_i), y_i)
+ \dis \paren*{\h \sP_\sX, \sfq}
+ \eta_\sH(S, \wt T)
+  2 M \Rad_{n}(\sH)
+ M \sqrt{\frac{\log \frac{1}{\beta}}{2n}}.
\end{equation}

The following more explicit upper bound on the Rademacher complexity holds when $\sH$ is the class of linear predictors and the support of $\sP_\sX$ is included in the $\ell_2$-ball
of radius $r$: $\Rad_{n}(\sH) \leq \sqrt{\frac{r^2 \Lambda^2}{n}}$
\cite{MohriRostamizadehTalwalkar2018}.
\cite{CortesMohri2014} proposed an adaptation algorithm motivated by
these learning bounds and other pointwise guarantees expressed in
terms of discrepancy. Their algorithm can be viewed as a two-stage
method seeking to minimize the first two terms of this learning bound. It
consists of first finding a minimizer $\sfq$ of
the weighted discrepancy (second term) and then minimizing 
(a regularized) $\sfq$-weighted empirical loss (first term) w.r.t. $h$ for that
value of $\sfq$.

We will design private adaptation algorithms
for a similar two-stage approach, as well as a single-stage
approach seeking to choose $h$ and $\sfq$ to directly
minimize the first two terms of the bound. The privacy and accuracy guarantees of our algorithms
crucially rely on a careful analysis of a smooth approximation of the discrepancy
term, which we present in the following section.

\ignore{
\thnote{Sections 3, 4, and 5 are fairly involved. I was wondering if we should given a high level view of our results here or earlier to help the readers?}\rnote{Again, I think that should be done in the intro.}
}

\section{Discrepancy analysis and smooth approximation}
\label{sec:smooth-approx}

\subsection{Analysis}

For the squared loss and $\sH = \curl*{x \mapsto w \cdot x \colon
  \norm{w} \leq \Lambda}$, the weighted discrepancy term of the
learning bound \eqref{eq:disbound} can be expressed in terms of the
spectral norm of a matrix that is an affine function of $\sfq$.

\begin{restatable}[\cite{MansourMohriRostamizadeh2009}]{lemma}{wdis}
  \label{lemma:wdis}
  For any distribution $\sfq$ over $S_\sX$, the following
  inequality holds:
\begin{align*}
\dis(\h P, \sfq)
& = 4 \Lambda^2 \norm{\bM(\sfq)}_2
= 4 \Lambda^2 \max\curl[\big]{\lambda_{\max} \paren*{\bM(\sfq)},
  \lambda_{\max} \paren*{-\bM(\sfq)}},
\end{align*}
where $\bM(\sfq) = \bM_0 - \sum_{i = 1}^m q_i \bM_i$ and where $\bM_0 =
\sum_{x \in \sX} \h \sP_\sX(x) xx^\top$, and $\bM_i = x_ix_i^\top$, $i
\in [m]$.
\end{restatable}
For completeness, the short proof is given in Appendix~\ref{app:disc-bounds}.
In view of that, the learning bound \eqref{eq:disbound}
suggests seeking $h \in \sH$ and $\sfq \in \Delta_m$ to minimize the
first two terms:
\begin{align}
\label{eq:one-stage-opt}
\mspace{-10mu}
\min_{\substack{h \in \sH\\ \sfq \in \Delta_m}} 
\sum_{i = 1}^m \sfq_i \ell(h(x_i), y_i)
+ 4 \Lambda^2 \norm{\bM(\sfq)}_2.
\end{align}
Note that the second term of the bound is sub-differentiable but it is
not differentiable both because of the underlying maximum operator and because
the maximum eigenvalue is not differentiable at points where its
multiplicity is more than one. Furthermore, the first term of the
objective function is convex with respect to $h$ and convex with
respect to $\sfq$, but it is not jointly convex.

The private algorithms we design require both the
smoothness of objective, which would not hold given the first issue
mentioned, and the sensitivity of the gradients. Thus, instead, we
will use a uniform $\alpha$-smooth approximation of the second term,
for which we analyze in detail the smoothness
and gradient sensitivity.

\subsection{Softmax smooth approximation}
\label{sec:soft-max}

A natural approximation of $\lambda_{\max}(\sfq)$ is based on the
softmax approximation:
$$F(\sfq) 
= \frac{1}{\mu} \log \bracket*{\sum_{i = 1}^d e^{\mu \lambda_i(\bM(\sfq))}}.$$
Note that while $F$ is a function of the eigenvalues, which are
not differentiable everywhere, it is in fact
infinitely differentiable since it 
can be expressed in terms of the trace of the exponential matrix
or the trace of powers of $\bM(\sfq)$:
$F(\sfq) = \frac{1}{\mu} \log \bracket[\Big]{\Tr\bracket[\big]{\exp\paren*{\mu
  \bM(\sfq)}}}$.
The matrix exponential can be computed in $O(d^3)$, using an SVD of
matrix $\bM(\sfq)$.  The following inequalities follow directly the
properties of the softmax:
\begin{align}
\label{ineq:approx-guarantee-F}
\mspace{-15mu}
\lambda_{\max}(\bM(\sfq)) & 
\mspace{-3mu}
\leq 
\mspace{-3mu}
F(\sfq) 
\mspace{-3mu}
\leq 
\mspace{-3mu}
\lambda_{\max}(\bM(\sfq)) \mspace{-3mu}
+
\mspace{-3mu}
\tfrac{\log(\rank(\bM(\sfq)))}{\mu}. 
\mspace{-10mu}
\end{align} 
Note that we have $\rank(\bM(\sfq))\leq \min\left(m+n,
d\right)$. Thus, for $\mu = \frac{\log(m+n)}{\alpha}$, $F(\sfq)$ gives
a uniform $\alpha$-approximation of $\lambda_{\max}(\bM(\sfq))$. Note that the components of the gradient of $F$ are given by
\begin{align}
  \label{eq:gradient_F}
   \bracket*{\nabla F(\sfq)}_j&= -\frac{\langle \exp\paren*{\mu \bM(\sfq)}, \bM_j\rangle}{\Tr\paren*{\exp\paren*{\mu \bM(\sfq)}}}, \qquad j \in [m]
\end{align}
where $\langle \cdot, \cdot \rangle$ denotes the Frobenius inner product. 
Both the smoothness and sensitivity of $\nabla F$ will be needed for
the derivation of our algorithm.  We now analyze these properties of
function $F$, using function $f$ which is defined for any symmetric
matrix $\bM \in \Sset_d$ as follows:
\begin{align*}
f(\bM) = \frac{1}{\mu} \log \bracket*{\sum_{k = 0}^{+\infty} \frac{\mu^k}{k!} \tri*{\bM^k, \bI} }.
\end{align*}
The following result provides the desired smoothness result needed
for $F$, which we prove by using the $\mu$-smoothness of $f$. 

\begin{restatable}{theorem}{ThFSmoothness}
\label{th:S-smoothness}
The softmax approximation
$F$ is $\mu \paren*{\max_{i \in [m]}
  \norm{x_i}_2^4}$-smooth for $\norm{\, \cdot \,}_1$.
\end{restatable}
The proof is given in Appendix~\ref{app:F}. Next, we analyze the
\emph{sensitivity of $\nabla F$}, that is the maximum variation
in $\ell_\infty$-norm of $\nabla F(\sfq)$ when a single point $x$ in
the sample of size $n$ drawn from $\h \sP_\sX$ is changed to another
one $x'$.

\begin{restatable}{theorem}{ThFSensitivity}
  \label{th:S-sensitivity}
The gradient of the softmax approximation $F$ is $\frac{2\mu r^2}{n}
\max_{i \in [m]} \norm*{x_i}^2_2$-sensitive.
\end{restatable}
The proof is given in Appendix~\ref{app:F}.

Note that the softmax function $f$ is known to be convex
\cite{BoydVandenberghe2014}. Since $\bM(\sfq)$ is an affine
function of $\sfq$ and that composition with affine functions
preserves convexity, this shows that $F$ is also a convex
function. The following further shows that $F$ is
$\max_{i \in [m]}\norm{x_i}^2_2$-Lipschitz.

\begin{restatable}{theorem}{ThSLipschitz}
  \label{th:S-lipschitz}
  For any $\sfq \in \Delta_m$, the gradient of
  $F$ is bounded as follows: $\norm{\nabla F(\sfq)}_\infty \leq
  \max_{i \in [m]}\norm{x_i}^2_2$.
\end{restatable}
The proof is given in Appendix~\ref{app:F}.
In view of the expression of the weighted discrepancy $\dis(\h \sP,
\sfq) = \max \curl*{\lambda_{\max}(\bM(\sfq)),
  \lambda_{\max}(-\bM(\sfq))}$, the smooth approximation $F(\sfq)$ of
the maximum eigenvalue of $\bM(\sfq)$ leads immediately to a smooth
approximation $\wt F(\sfq) = f(\wt \bM(\sfq))$ of $\dis(\h \sP,
\sfq)$, with
\[
\wt \bM(\sfq)
=
\begin{bmatrix}
  \bM(\sfq) & \bzero\\
  \bzero & -\bM(\sfq)
\end{bmatrix}.
\]
Thus, $\wt F$ inherits the key properties of $F$
gathered in the following corollary.

\begin{restatable}{corollary}{wtF}
  \label{cor:wtf}
  The following properties holds for $\wt F$:
\begin{enumerate}[topsep=0mm, itemsep=0mm]
\item $\wt F$ is convex and is a uniform $\frac{\log(2 \min\curl*{m + n,
    d})}{\mu}$-approximation of $\sfq \mapsto \dis(\h P, \sfq)$.

\item $\wt F$ is $\mu
\paren*{\max_{i \in [m]} \norm{x_i}_2^4}$-smooth for $\norm{\, \cdot
  \,}_1$.

\item $\norm{\nabla \wt F}_\infty$ is 
$\frac{2\mu r^2}{n}
  \max_{i \in [m]} \norm*{x_i}^2_2$-sensitive.

\item for any $\sfq \in \Delta_m$, $\norm{\nabla \wt
  F(\sfq)}_{\infty} \leq \max_{i \in [m]}\norm{x_i}^2_2$.
\end{enumerate}
\end{restatable}
The proof is given in Appendix~\ref{app:wtF}.
In Appendix~\ref{app:G}, we also present and analyze a $p$-norm smooth approximation
of the discrepancy. This approximation can be used to design private adaptation
algorithms with a relative deviation guarantee that can be more favorable in some
contexts.

\section{Two-stage private adaptation algorithms}
\label{sec:two-stage}

Here, we discuss private solutions for a two-stage approach that
consists of first finding $\sfq$ that minimizes the empirical
discrepancy and next fixing $\sfq$ to that value and minimizing the
empirical $\sfq$-weighted loss over $h \in \sH$. In the absence of
privacy constraints, this coincides with the two-stage algorithm of
\cite{CortesMohri2014}.
The first stage consists of seeking $\sfq \in \Delta_m$ to minimize an
$\ell_2$-regularized version of the discrepancy, the second stage simply
consists of fixing the solution $\sfq$ obtained in the first stage and
seeking $h \in \sH$ minimizing the $\sfq$-weighted empirical loss:
\begin{align}
\label{eq:one-stage-first}
\min_{\sfq \in \Delta_m} \norm{\bM(\sfq)}_2 +\frac{\lambda}{2} \norm{\sfq}_2^2
\qquad \quad
\min_{w \in \mathbb{B}_2^d(\Lambda)} \sum_{i = 1}^m \sfq_i \ell(\tri{w, x_i}, y_i),
\end{align}
where $\mathbb{B}_2^d(\Lambda)$ is the Euclidean ball in $\Rset$ of radius $\Lambda$. Equivalently, we can define an $\ell_2$-regularized version of the weighted empirical loss and minimize it over $\Rset^d$; namely, solve
$$\min_{w \in \Rset^d} \sum_{i = 1}^m \sfq_i \ell(\tri{w, x_i}, y_i) + \wt{\lambda} \norm{w}^2_2,$$
where $\wt{\lambda} > 0$ is a hyperparameter. Regularization in the first stage is done to ensure that the resulting weights $\sfq$ are not too
sparse since sparse solutions can lead to poor output model in the second
stage of the adaptation algorithm. 

In the second stage, no
private data is involved. Thus, in this section, we
give two private algorithms for the first stage of discrepancy
minimization. Our private algorithms aim at minimizing an $\ell_2$-regularized version of the
smooth approximation, $\wt F$, of the discrepancy discussed in
Section~\ref{sec:soft-max}. To emphasize its dependence on the private
unlabeled dataset $T$, we will use the notation $\wt F_T$. Namely, our
algorithms aim at privately minimizing an $\ell_2$-regularized version of
$\wt F_T$:
\[
\Freg\triangleq \wt F_T(\sfq)+\frac{\lambda}{2}\norm{\sfq}_2^2.
\]
As mentioned earlier, the regularization term is used to avoid sparse
solutions $\sfq$ that may impact the accuracy of the output model in
the second stage of the adaptation algorithm.   
Our algorithms are based on private variants of the Frank-Wolfe algorithm
and the Mirror Descent algorithm. The general structure of these algorithms follow known private constructions devised in the context of differenitally private empirical risk minimization \cite{talwar2015nearly, bassily2021non,AsiFeldmanKorenTalwar2021}.
However, we note that the
guarantees of both algorithms crucially rely on the smoothness and sensitivity
properties of the approximation proved in the previous section. 
Solving the optimization with respect to the smooth approximation
of the discrepancy enables us to bound the sensitivity of the
gradients (see Theorem~\ref{th:S-sensitivity}), which helps us 
devise private solutions for this problem. 

 We defer the description of these algorithms to Appendix~\ref{app:two-stage}. We state below their formal guarantees.  

\begin{restatable}{theorem}{FWmain}
\label{th:FWmain}
The Noisy Frank-Wolfe algorithm (Algorithm~\ref{Alg:nfw} in Appendix~\ref{app:PrivateFW}) is $(\varepsilon, \delta)$-differentially private. Let $\sfq^\ast\in\argmin_{\sfq\in \Delta_m}\dis(\h P,
\sfq)$. There exists a choice of the parameters of Algorithm~\ref{Alg:nfw} such that with high probability over the algorithm's internal randomness, the output $\widehat{\sfq}$ satisfies
\begin{align*}
  \dis(\h P, \widehat{\sfq})
  \leq& \dis(\h P, \sfq^\ast) + \frac{\lambda}{2} \norm{\sfq^\ast}_2^2+ \wt{O}\paren*{\frac{1}{(\varepsilon n)^{1/3}}}.
\end{align*}
\end{restatable}
The smoothness we created in $\wt F_T$ also enables us to use a private variant of the Frank-Wolfe algorithm, whose optimization error scales
only logarithmically with $m$.

\begin{restatable}{theorem}{MDmain}
\label{th:MDmain}
The Noisy Mirror Descent algorithm (Algorithm~\ref{Alg:nmd} in Appendix~\ref{app:PrivateMD}) is $(\varepsilon, \delta)$-differentially private. Let $\sfq^\ast\in\argmin_{\sfq\in \Delta_m}\dis(\h P,
\sfq)$. There exists a choice of the parameters of Algorithm~\ref{Alg:nmd} such that with high probability over the algorithm's randomness, the output $\widehat{\sfq}$ satisfies
\begin{align*}
  \dis(\h P, \widehat{\sfq})
  \leq& \dis(\h P, \sfq^\ast) + \frac{\lambda}{2} \norm{\sfq^\ast}_2^2+ \wt{O}\paren*{\frac{m^{1/4}}{\sqrt{\varepsilon n}}}.
\end{align*}
\end{restatable}
Note that compared to
the guarantees of the private Frank-Wolfe algorithm in Theorem~\ref{th:FWmain}, the optimization error of the Noisy Mirror Descent algorithm (Theorem~\ref{th:MDmain}) exhibits a better
dependence on $n$ at the expense of worse dependence on $m$. In Appendix~\ref{app:two-stage}, we give full proofs of these theorems.

Note that by standard
stability arguments, the minimum weighted empirical loss of the second
stage when training with $\sfq^\ast$ is close to the minimum weighted empirical loss when
training with $\h \sfq$ when the discrepancy between $\h \sfq$ and $\sfq^*$ 
is small \cite{MansourMohriRostamizadeh2009}\ignore{[Theorem~11]}.
Theorems~\ref{th:FWmain} and \ref{th:MDmain} precisely supply guarantees for that closeness
in discrepancy via the inequality $\dis(\h \sfq, \sfq^\ast)
\leq \dis(\h \sP, \h \sfq) - \dis(\h \sP, \sfq^\ast) $, thereby
guaranteeing the closeness of the loss of our  private predictor (output of the second stage) to the minimum $\sfq^{\ast}$-weighted empirical loss. This together with the learning bound (\ref{eq:disbound}) immediately provide a bound on the expected loss of our private predictor.



\section{Single-stage private adaptation algorithm}
\label{sec:single-stage}

In this section, we give a novel  private algorithm for that outputs an approximate stationary point of the smooth approximation of the learning bound: $L_T(\sfq, w)\triangleq \sum_{i=1}^m
\sfq_i(\langle w, x_i\rangle -y_i)^2+4\Lambda^2 \wt{F}_T(\sfq).$ 
    Here, $\wt F_T(\sfq),~\sfq\in \Delta_m,$
    is the smooth approximation of the discrepancy discussed in
    Section~\ref{sec:soft-max} (where the subscript $T$ in $L_T$ and
    $\wt F_T$ is used to emphasize the dependence on the private
    dataset $T$).

As discussed earlier, the function $L_T$ is generally non-convex in
$\sfq, w$. Since attaining a global minimizer of $L_T$ is generally
intractable, a reasonable alternative is to find (an approximate)
stationary point of $L_T$.  Note that $L_T(\sfq, w)$ is smooth in
$\sfq$ w.r.t. $\norm{\cdot}_1$ (as discussed in
Section~\ref{sec:soft-max}) and smooth in $w$ w.r.t. $\norm{\cdot}_2$
(due to the nature of the squared loss). These smoothness properties
allow us to design our private solution. Given the approximation
guarantee (\ref{ineq:approx-guarantee-F}), the data-dependent terms in
the learning bound (\ref{eq:disbound}) can thus be approximated by
$L_T(\sfq, w)$. Hence, our strategy here is to find an approximate
stationary point $(\h{\sfq}, \h{w})$ of $L_T$ via our private
algorithm, and then derive a learning bound in terms of $L_T(\h{\sfq},
\h{w})$. The formal definition of an approximate stationary point is
given next.

\begin{defn}[$\alpha$-approximate stationary point]\label{def:st-gap}
Let $f:\cC\rightarrow \Rset$ be a differentiable function over a
convex and compact subset $\cC$ of a normed vector space. Let
$\alpha\geq 0$. We say that $u\in\cC$ is an $\alpha$-approximate
stationary point of $f$ if the stationarity gap of $f$ at $u$, defined
as $\gap_f(u)\triangleq \max\limits_{v\in \cC}\langle -\nabla f(u),
v-u\rangle$ is at most $\alpha$.
\end{defn}

First, we will give a generic differentially-private algorithm for
approximating a stationary point of smooth non-convex objectives
$f_T:\cQ\times\cW\rightarrow \Rset$ (defined by a private dataset $T$)
that satisfy certain smoothness and Lipschitzness conditions. We give
formal definitions of these conditions below.

\begin{defn}[$\paren*{(\gamma_\sfq, \norm{\cdot}_{p_1}), (\gamma_w, \norm{\cdot}_{p_2})}$-Lipschitz function]
Consider a function $f \colon \cQ\times \cW\rightarrow\Rset$, where
$\cQ$ is a convex set whose $\norm{\cdot}_{p_1}$-diameter is bounded
by $D_\sfq$ (we refer to $\cQ$ as $(D_\sfq,
\norm{\cdot}_{p_1})$-bounded set), and $\cW$ is a convex $(D_w,
\norm{\cdot}_{p_2})$-bounded set.  Let $\gamma_\sfq, \gamma_w\geq
0$. We say that $f$ is $\paren*{\paren*{\gamma_\sfq, \norm{\cdot}_{p_1}},
  \paren*{\gamma_w, \norm{\cdot}_{p_2}}}$-Lipschitz if for any $w\in\cW$,
$f(\cdot, w)$ is $\gamma_\sfq$-Lipschitz w.r.t.\ $\norm{\cdot}_{p_1}$
over $\cQ$, and for every $\sfq\in \cQ$, $f(\sfq, \cdot)$ is
$\gamma_w$-Lipschitz w.r.t.\ $\norm{\cdot}_{p_2}$ over $\cW$.
\end{defn}

\begin{defn}[$\paren*{(\mu_\sfq, \norm{\cdot}_{p_1}), (\mu_w, \norm{\cdot}_{p_2})}$-smooth function]
This notion is defined analogously. We say that $f$ is
$\paren*{(\mu_\sfq, \norm{\cdot}_{p_1}), (\mu_w,
  \norm{\cdot}_{p_2})}$-Lipschitz if for any $w\in\cW$, $f(\cdot, w)$
is $\mu_\sfq$-smooth w.r.t. $\norm{\cdot}_{p_1}$ over $\cQ$, and for
every $\sfq\in \cQ$, $f(\sfq, \cdot)$ is $\mu_w$-Lipschitz
w.r.t. $\norm{\cdot}_{p_2}$ over $\cW$.
\end{defn}

Our private algorithm (Algorithm~\ref{Alg:stFW} below) takes as input an objective $f_T:\cQ\times \cW\rightarrow \Rset$, where $\cQ$
is a convex polyhedral set with bounded $\norm{\cdot}_1$-diameter and
$\cW$ is a convex set with bounded $\norm{\cdot}_2$-diameter. Hence,
our objective $L_T$ mentioned earlier is a special case. The algorithm
is comprised of a number rounds, where in each round, two private
Frank-Wolfe update steps are performed; one for $\sfq$ and another for
$w$. The privacy mechanism for each is different due to the different
geometries of $\cQ$ and $\cW$. We note that in the special case where
$f_T=L_T$, there is no need to privatize the Frank-Wolfe step for $w$
due to the fact that such update step depends only on the $\sfq$-weighted empirical loss over the public data and the fact that differential privacy is closed under
post-processing (the previous update step for $\sfq$ is carried out in a differentially private manner).

When $f_T$ satisfies the Lipschitzness and smoothness
properties defined above w.r.t. $\norm{\cdot}_1$ and $\norm{\cdot}_2$,
we give formal convergence guarantees to a stationary point in terms
of a high-probability bound on the \emph{stationarity gap} of the
output (see Definition~\ref{def:st-gap}). Despite the different geometries of $\cQ$ and $\cW$, our final
bound is roughly the sum of the bounds we would obtain if we ran two
separate Frank-Wolfe algorithms (one over $\cQ$ and the other over
$\cW)$. This is mainly due to the hybrid Lipschitzness and smoothness
conditions ($\norm{\cdot}_1$ for $\sfq$ and $\norm{\cdot}_2$ for $w$),
which enable us to decompose the bound on the convergence rate over
$\sfq$ and $w$.

\begin{algorithm}[t]
    \caption{Private Frank-Wolfe for approximating  stationary points of $f_T:\cQ\times\cW\rightarrow \Rset$}
    \begin{algorithmic}[1]
    \REQUIRE Private dataset: $T=(z_1,\ldots, z_n)\in \cZ^n$,  privacy parameters $(\varepsilon, \delta)$, a convex $(D_\sfq, \norm{\cdot}_1)$-bounded polyhedral set: $\cQ\subset \Rset^m$ with $J$ vertices $\cV= (v_1, \ldots, v_J)$, a convex $(D_w, \norm{\cdot}_2)$-bounded set $\cW \subset \Rset^d$, a function $f_T(\sfq, w), ~\sfq\in \cQ, w\in \cW$ (defined via the dataset $T$), bound on the global $\norm{\cdot}_\infty$-sensitivity of $\nabla_\sfq f_T(\sfq, w):~\tau_\sfq>0$, bound on the global $\norm{\cdot}_2$-sensitivity of $\nabla_w f_T(\sfq, w):~\tau_w\geq 0$, step size: $\eta$, number of iterations: $K$.
    \STATE Set $\sigma_\sfq:= \frac{4\tau_\sfq\sqrt{2K\log(\frac{1}{\delta})}}{\varepsilon}$.\label{step:sigma_q}
    \STATE Set $\sigma_w:= \frac{4\tau_w\sqrt{2K\log(\frac{1}{\delta})}}{\varepsilon}$.\label{step:sigma_w}
    \STATE Choose arbitrarily $(\sfq^0, w^0)\in \cQ\times \cW$.
    \FOR{$k=0$ to $K-1$}
    \STATE $\nabla_\sfq^{k}:=\nabla_\sfq f_T(\sfq^{k}, w^{k})$. 
    \STATE Draw $\paren*{b_v^k: v\in \cV}$ independently $\sim \lap(\sigma_\sfq)$.
    \STATE $v_\sfq^{k}:=\argmin\limits_{v\in\cV}\curl[\Big]{\langle \nabla_\sfq^k, v\rangle + b_v^k}$.\label{step:stFW-report-noisy-min1}
    \STATE $G_\sfq^k:= -\paren*{\langle \nabla_\sfq^k, v_\sfq^k-\sfq^k\rangle + b_v^k}$.\label{step:stFW-report-noisy-min2}
    \STATE $\sfq^{k+1}:= (1-\eta)\sfq^{k}+\eta v_\sfq^{k}.$
    \STATE $\nabla_w^{k}:=\nabla_w f_T(\sfq^k, w^k)$. 
    \STATE $\hnabla_w^{k}:=\nabla_w^{k}+\bg^k$, where $\bg^k\sim \cN(\mathbf{0}, \sigma_w^2\mathbb{I}_d)$.\label{step:gauss-mech} 
    \STATE $u_w^{k}:= \argmin\limits_{u\in \cW}\langle \hnabla_w^k, u \rangle$. \label{step:u_w_k} 
    \STATE $G_w^k:= - \langle \hnabla_w^k, u_w^k-w^k\rangle.$
    \STATE $w^{k+1}:= (1-\eta)w^{k}+\eta u_w^{k}.$ 
    \ENDFOR
    \RETURN $(\h{\sfq}, \h{w})=(\sfq^{k^\ast}, w^{k^{\ast}})$, where ~ $k^\ast=\argmin\limits_{k\in [K]}(G_\sfq^k+G_w^k).$
   \end{algorithmic}
    \label{Alg:stFW}
\end{algorithm}

\begin{restatable}{theorem}{stFW}
\label{th:stFW}
Algorithm~\ref{Alg:stFW} is $(\varepsilon, \delta)$-differentially
private.\ignore{  Let $\gamma_\sfq, \gamma_w, \mu_\sfq, \mu_w, \gamma_{\sfq,
  w}\geq 0$.} Assume that the objective $f_T \colon \cQ\times
\cW\rightarrow\Rset$ is $\paren*{(\gamma_\sfq, \norm{\cdot}_1),
  (\gamma_w, \norm{\cdot}_2)}$-Lipschitz and $\paren*{(\mu_\sfq,
  \norm{\cdot}_1), (\mu_w, \norm{\cdot}_2)}$-smooth. Assume further
that for all $\sfq\in\cQ,$ and $w, w'\in \cW$,
\mbox{$\norm{\nabla_\sfq f_T(\sfq, w)-\nabla_\sfq f_T(\sfq,
    w')}_\infty \leq \gamma_{\sfq, w}\norm{w-w'}_2$}. Then, for any $\beta\in
(0, 1)$, there exists a choice of $K$ and $\mu$ such that with probability at
least $1 - \beta$ (over the algorithm's randomness),
the stationarity gap of the output $\h w$ is upper bounded by
\begin{align*}
   \gap_{f_T}(\h{\sfq}, \h{w})&\leq 5\sqrt{\Bar{D}\paren*{\sigma^0_\sfq\log\paren*{\tfrac{\Bar{D}J}{\sigma^0_\sfq \beta}}+D_w\sigma^0_w\sqrt{d\log\paren*{\tfrac{\Bar{D}}{D_w\sigma^0_w \beta}}}}}.
\end{align*}
\end{restatable}
The proof is given in Appendix~\ref{app:stFW}. We note that our
adaptation objective $L_T(\sfq, w)$ satisfies all the conditions in
Theorem~\ref{th:stFW}. In Appendix \ref{app:stFW-adap}, we give a
detailed discussion regarding instantiating Algorithm~\ref{Alg:stFW} with
$L_T(\sfq, w)$ and the specific settings of all the parameters in this
special case. As a result, we immediately reach the following
corollary.
\begin{restatable}{corollary}{stFWadap}
\label{th:stFWadap}
Let $L_T(\sfq, w)=\sum_{i=1}^m \sfq_i(\langle w, x_i\rangle
-y_i)^2+4\Lambda^2 \wt{F}_T(\sfq)$ be the input to
Algorithm~\ref{Alg:stFW}. Let $\beta\in (0, 1)$. There exists a choice
of $K$ and $\eta$ fow which, with probability at least $1 - \beta$,
the output of the algorithm is an approximate stationary point of
$L_T$ with stationarity gap upper
bounded as 
$$\gap_{L_T}(\h{\sfq}, \h{w})\leq \wt O\paren*{\frac{\mu^{3/4}}{\sqrt{\varepsilon
    n}}}.$$  
Hence, bound (\ref{eq:disbound}) implies that w.p. $\geq
1-2\beta$ over the choice of the public and private datasets and the
algorithm's internal randomness, the expected loss of the predictor
$h_{\h{w}}$ (defined by the output $\h{w}$) w.r.t. the target domain
is bounded as
\begin{align*}
\cL(\sP, h_{\h{w}})
\leq& L_T(\h{\sfq}, \h{w})+\wt O\paren*{\frac{1}{\mu}
+  \frac{1}{\sqrt{n}}}
+\eta_\sH(S, \wt T).
\end{align*}
\end{restatable}
\begin{remark*}
Note that $(\h{\sfq}, \h{w})$ is an approximate stationary point of $L_T$. In practice, $(\h{\sfq}, \h{w})$ can be an approximation of a good local minimum of $L_T$ as demonstrated by our experiments. In such situations, the above bound implies a good prediction accuracy for the output predictor. Note also that the bound above is given in terms of the soft-max
approximation parameter $\mu$. In general, this parameter should be
treated as a hyper-parameter and tuned appropriately to minimize the
above bound. One reasonable choice of $\mu$ can be obtained by
balancing the bound on the stationarity gap with the error term
$\log(m+n)/\mu$ due to the soft-max approximation. In such case,
$\mu=\wt O\paren{(\varepsilon n)^{2/7}}$.
\end{remark*}

\section{Experiments}
\label{sec:experiments}

\begin{figure}[t]
    \centering
    \vskip -.15in
    \begin{tabular}{cc}
    \includegraphics[scale=0.35]{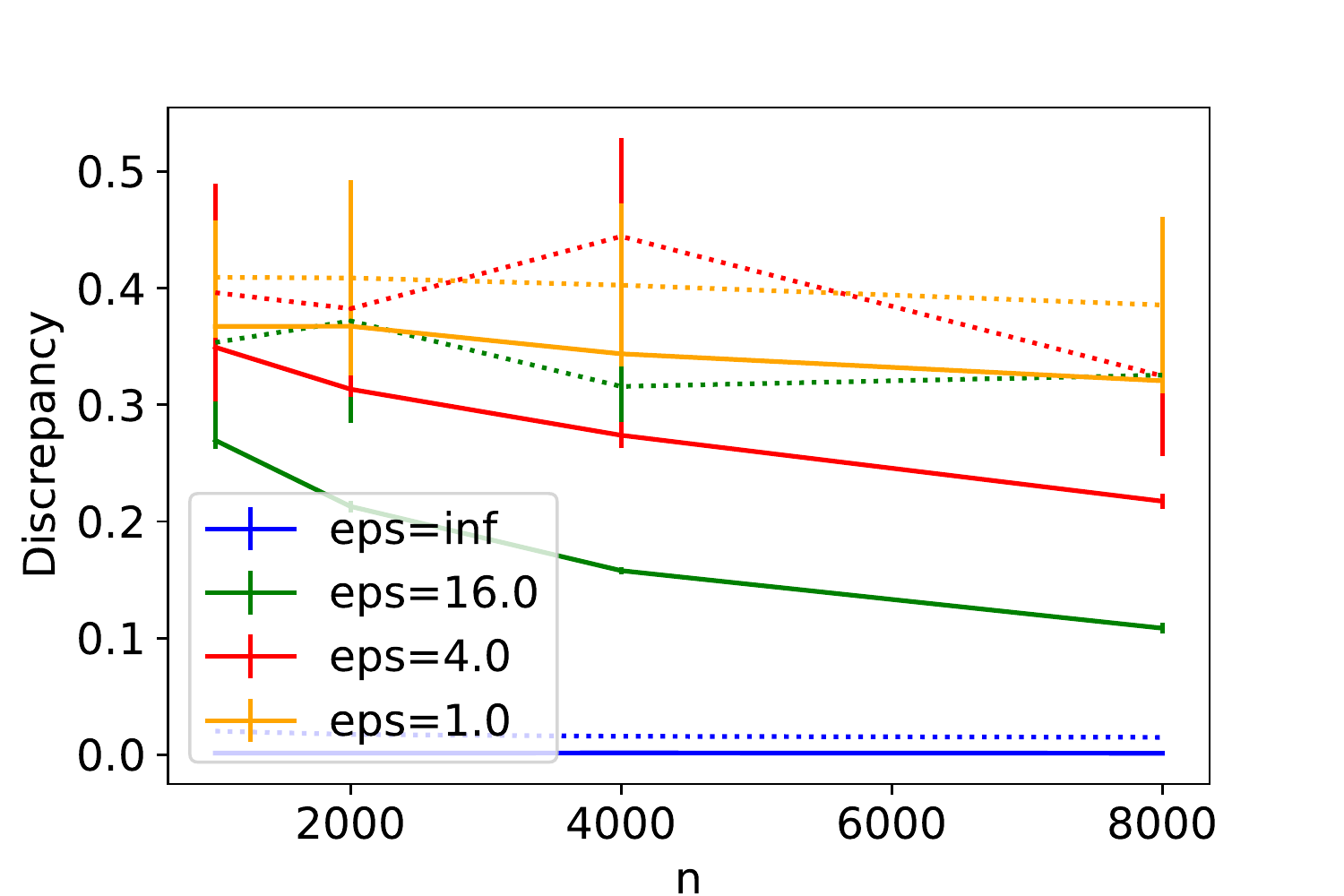}
    & \includegraphics[scale=0.23]{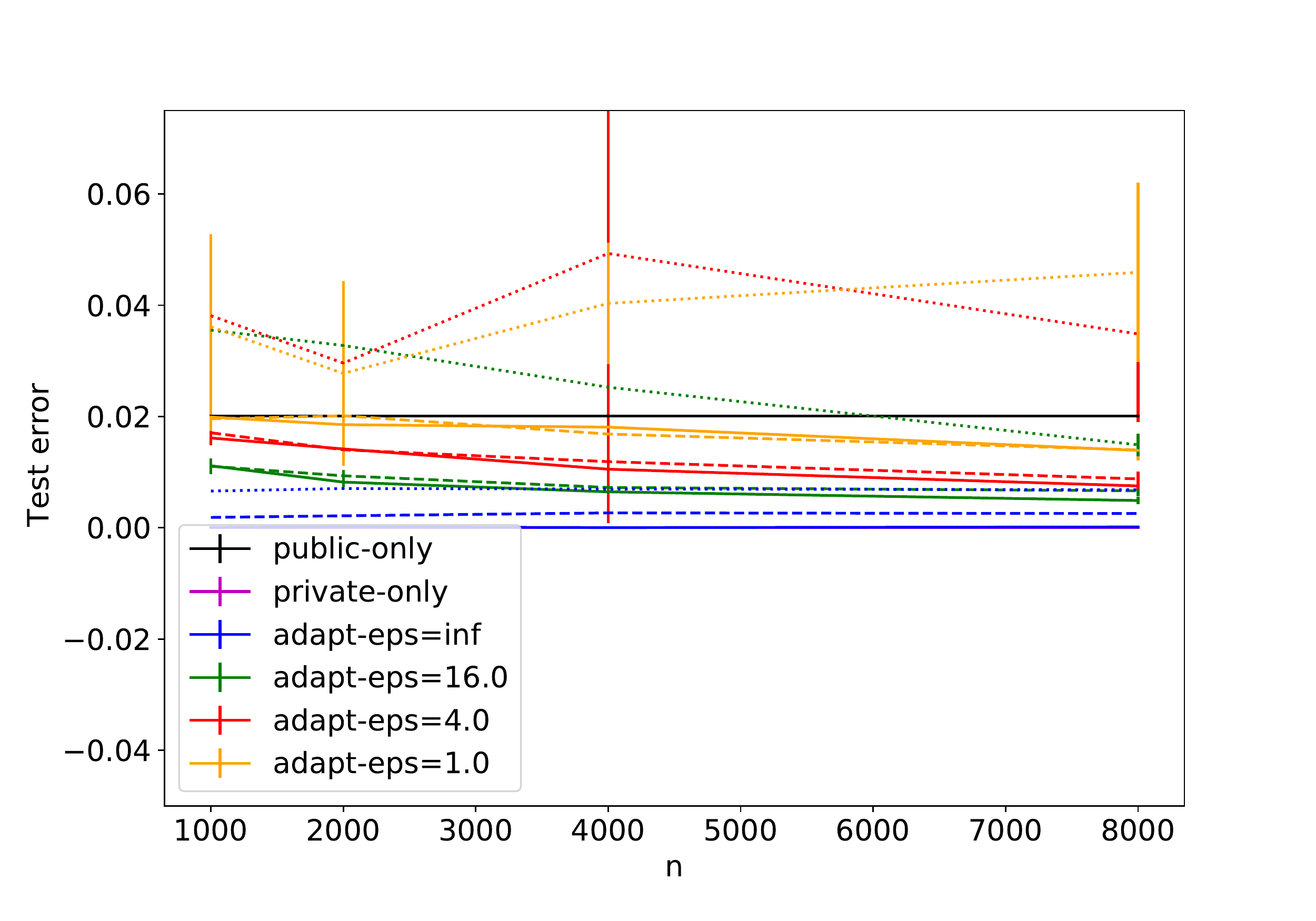} \\
    (a) & (b)\\
    \end{tabular}
        \caption{(a) Value of the spectral norm $\norm{\bM(\sfq)}_2$ for the output of noisy Frank-Wolfe (solid lines) and noisy Mirror descent (dotted lines) discrepancy minimization 
    as a function of the number of samples from the private dataset $n$. (b) Test error as a function of the number of samples from the private dataset $n$. 
      The solid lines correspond to the single-stage algorithm, the dotted lines to the two-stage mirror decent algorithm, and dashed lines to the two-stage Frank Wolfe algorithm.}
    \vskip -.15in
    \label{fig:gauss_lambda}
\end{figure}

The objective of this section is to
provide proof-of-concept experiments to
demonstrate that reasonable privacy guarantees could be achieved, when
using our private domain adaptation algorithms. We use a setting
similar to that of  \cite[Section~7.1]{CortesMohri2014} and
demonstrate that the utility of private adaptation degrades gracefully
with increased privacy guarantees and that the single-stage Frank-Wolfe algorithm
performs best 
in most scenarios.

We carried out experiments with the following synthetic dataset. Let
$d=10$ and $\sigma^2 = 1/(9d)$.  Let $P_\sX$ be a spherical Gaussian
centered around $(-1/\sqrt{2d}, 1/\sqrt{2d}, \ldots, -1/\sqrt{2d},
1/\sqrt{2d})$ and with variance $\sigma^2$ in all directions. Let
$Q_\sX$ be a Gaussian distribution with mean $(1/\sqrt{2d}, \ldots,
1/\sqrt{2d})$ and with variance $\sigma^2$ in all directions. We
defined the labeling function via $f(x) = x \cdot \bar{1}$ if $\bar{1}
\cdot x > 0$, $\left(\frac{1}{2} x\cdot \bar{1}\right)$ otherwise,
where $\bar{1} = (1/\sqrt{d}, \ldots, 1/\sqrt{d})$. We chose the
target distribution to be $P_{\sX}$ and the source distribution as a
mixture of $P_{\sX}$ and $Q_{\sX}$ with the weight of $P_{\sX}$ set to
$25\%$. We fixed the number of source samples to be $1\mathord{,}000$
and varied the number of unlabeled target samples from
$1\mathord{,}000$ to $8\mathord{,}000$. All experiments were repeated
ten times for statistical consistency. We set $K=1\mathord{,}000$, $\lambda = 0.001$, the privacy parameter
$\delta = 1/8\mathord{,}000$, and varied $\e$ in experiments.  The
standard deviations were calculated over $10$ runs in experiments.

In this setup, we first ran differentially private discrepancy
minimization using Algorithms~\ref{Alg:nfw} and~\ref{Alg:nmd}. We plotted
$\norm{\bM(\sfq)}_2$ for different values of $\e$ in
Figure~\ref{fig:gauss_lambda}(a). The performance of the noisy Frank-Wolfe algorithm degrades smoothly with $\e$ and improves with $n$. However the performance of the noisy mirror decent algorithm is much worse. This is in line with the theoretical guarantees as $m = \Omega(n^{2/3})$ in these experiments and noisy Frank-Wolfe algorithm has a better convergence guarantee in this regime. We expect mirror descent to perform better with much larger values of $n$. Furthermore, observe that the noisy mirror descent has a high standard deviation compared to Frank-Wolfe algorithm as the noise added in mirror descent scales polynomially in $m$, whereas it scales only logarithmically in $m$ for the Frank-Wolfe algorithm. 

We next compared our single-stage (Algorithm~\ref{Alg:stFW}) and the two-stage differentially private
algorithms with the model trained only with the public
dataset (Figure~\ref{fig:gauss_lambda}(b)). As an oracle baseline, we
also plotted the model trained with the labeled private dataset. Note
that this model uses extra information that is not available during
training and is plotted for illustration purposes only. The
single-stage Frank-Wolfe algorithm without privacy admits the same
performance as the model trained on the labeled private dataset. It
performs better than the two-stage Frank-Wolfe algorithm, however the gap
decreases as the privacy guarantee $\epsilon$ improves. The performance of the mirror descent algorithm without differential privacy is similar to that of Frank-Wolfe algorithm, however as theory indicates, the performance degrades quickly with the privacy parameter. Similar to Figure~\ref{fig:gauss_lambda}(a), the performance of the noisy mirror descent algorithm is much worse and has a high standard deviation.

\ignore{  We did not
run experiments with the mirror descent algorithm as it is relatively
more expensive computationally since its gradient complexity scales
quadratically with the size of the private sample.}

\section{Conclusion}

We presented new differentially private adaptation algorithms
benefitting from strong theoretical guarantees. Our analysis can form
the basis for the study of privacy for other related adaptation
scenarios, including scenarios where a small amount of (private)
labeled data is also available from the target domain and those with
multiple sources. Our single-stage private algorithm is further likely
to be of independent interest for private optimization of other
similar objective functions.

\section*{Acknowledgements}\label{sec:ack}
This work was done while RB was visiting Google, NY. 
RB's research at OSU is supported by NSF Award AF-1908281, NSF Award 2112471, Google Faculty Research Award, and NSF CAREER Award 2144532.

\newpage

\bibliographystyle{alpha}
\bibliography{dpa}

\newpage

\appendix

\renewcommand{\contentsname}{Contents of Appendix}
\tableofcontents
\addtocontents{toc}{\protect\setcounter{tocdepth}{4}} 
\clearpage

\section{Background on discrepancy-based generalization bounds}
\label{app:background}

In this section, we briefly present some background material
on discrepancy-based generalization guarantees.

The following learning bound was given by
\cite{CortesMohriMunozMedina2019}: for any $\beta > 0$, with probably
at least $1 - \beta$ over the draw of a sample $S \sim \sQ^m$, for any
distribution $\sfq$ over $S_\sX$, for all $h \in \sH$, the following
inequality holds:
\begin{align}
\label{eq:disYbound}
\cL(\sP, h)
&\leq \sum_{i = 1}^m \sfq_i \ell(h(x_i), y_i)
+ \dis_\sY \paren*{\h \sP, \sfq}
+  2 \Rad_{n}(\ell \circ \sH)
+ M \sqrt{\frac{\log \frac{1}{\beta}}{2n}}.
\end{align}
This bound is tight in the sense that for the hypothesis reaching
the maximum in the definition of the $\sY$-discrepancy, the bound
coincides with the standard Rademacher complexity bound on $\h \sP$
\cite{CortesMohriMunozMedina2019}.
The bound suggests choosing $h \in \sH$ and the distribution $\sfq$ to
minimize the right-hand side. The first term of the bound is not
jointly convex with respect to $h$ and $\sfq$.  Instead, the algorithm
suggested by \cite{CortesMohri2014} (see also
\cite{MansourMohriRostamizadeh2009}) consists of a two-stage
procedure: first choose $\sfq$ to minimize the $\sfq$-weighted
empirical discrepancy, next fix $\sfq$ and choose $h$ to minimize the
$\sfq$-weighted empirical loss $\sum_{i = 1}^m \sfq_i \ell(h(x_i),
y_i)$.

In practice, we do not have labeled data from $\sP$ or too few to be
able to accurately minimize the $\sY$-discrepancy, since otherwise
adaptation would not be even necessary and we could directly use
labeled data from $\sP$ for training. Instead, we upper bound the
$\sY$-discrepancy in terms of the discrepancy $\dis(\sP_\sX, \sfq)$
and the \emph{output label-discrepancy} $\eta_\sH(S, \wt T)$ defined as
follows:
\[
\eta_\sH(S, \wt T) = \min_{h_0 \in \sH} \curl*{ 
\sup_{(x, y) \in S} |y - h_0(x)|
+ \sup_{(x, y) \in \wt T} |y - h_0(x)|
},
\]
where $\wt T$ is the labeled version of $T$ (i.e., $\wt T$ is $T$ associated with its true (hidden) labels). Note that $\dis(\sP_\sX, \sfq)$ measures the difference of the
distributions on the input domain.
In contrast, $\eta_\sH(S, \wt T)$ accounts for the difference of the
output labels in $S$ and $T$. We will assume that $\eta_\sH(S, \wt T)\ll 1$. 
Note that under the covariate shift assumption and separable case, we have
$\eta_\sH(S, \wt T) = 0$. In general, adaptation is not possible when
$\eta_\sH(S, \wt T)$ can be large since the labels received on the training
sample 
can be different from the target ones.

We will say that a loss function $\ell$ is
\emph{$\gamma$-admissible} if $\abs*{\ell(h(x), y) - \ell(h'(x), y)}
\leq \gamma \abs*{h(x) - h'(x)}$ for all $(x, y) \in \sX \times \sY$
and $h \in \sH$ \cite{CortesMohriMunozMedina2019}. Note that this is
a slightly weaker condition than that of $\gamma$-Lipschitzness of the
loss with respect to its first argument.

\begin{restatable}{theorem}{DisYUpperBound}
\label{th:DisYUpperBound}
  Let $\ell$ be a $\gamma$-admissible loss. Then, the following upper
  bound holds:
  \[
\dis_\sY(\sP, \sQ)
\leq \dis(\sP_\sX, \sfq) + \gamma \, \eta_\sH(\supp(\sP), \supp(\sQ)).
  \]
\end{restatable}
The proof is given in Appendix~\ref{app:disc-bounds}.
Note that the squared loss is $2M$-admissible: since the function
$x \mapsto x^2$ is 2-Lipschitz on $[0, 1]$, we have
$\abs*{\ell(h(x), y) - \ell(h'(x), y)}
= M \abs*{\frac{\ell(h(x), y)}{M} - \frac{\ell(h'(x), y)}{M}}
\leq  2M \abs*{h(x) - h'(x)}$. 
%
Thus, the learning bound \eqref{eq:disYbound} can be expressed in
terms of the discrepancy and the Rademacher complexity of $\sH$ as
follows, using the fact $\Rad_{n}(\ell \circ \sH) \leq
2M\Rad_{n}(\sH)$ \cite{MohriRostamizadehTalwalkar2018}[Prop. 11.2]:
\begin{align*}
\cL(\sP, h)
&\leq \sum_{i = 1}^m \sfq_i \ell(h(x_i), y_i)
+ \dis \paren*{\h \sP_\sX, \sfq}
+ \eta_\sH(S, S')
+  2 M \Rad_{n}(\sH)
+ M \sqrt{\frac{\log \frac{1}{\beta}}{2n}}.
\end{align*}
We will be considering a family of linear hypotheses $\sH = \curl*{x
  \mapsto w \cdot x \colon \norm{w} \leq \Lambda}$ and will be
assuming that the support of $\sP_\sX$ is included in the $\ell_2$ ball of
radius $r$.
The following more explicit upper bound on the Rademacher complexity then holds when the support of $\sP_\sX$ is included in the $\ell_2$-ball
of radius $r$: $\Rad_{n}(\sH) \leq \sqrt{\frac{r^2 \Lambda^2}{n}}$
\cite{MohriRostamizadehTalwalkar2018}.
\ignore{\cite{CortesMohri2014} proposed an adaptation algorithm motivated by
these learning bounds and other pointwise guarantees expressed in
terms of discrepancy. Their algorithm can be viewed as a two-stage
method seeking to minimize the first two terms of this learning bound. It
consists of first finding a minimizer $\sfq$ of
the weighted discrepancy (second term) and then minimizing 
(a regularized) $\sfq$-weighted empirical loss (first term) w.r.t. $h$ for that
value of $\sfq$.}

\ignore{We will design private adaptation algorithms
for a similar two-stage approach, as well as a single-stage
approach seeking to choose $h$ and $\sfq$ to directly
minimize the first two terms of the bound. The design of our algorithms
is based on an analysis of a smooth approximation of the discrepancy
term, which we discuss in the following section. \rnote{this last paragraph doesn't seem to belong here.}}

\ignore{
\thnote{Sections 3, 4, and 5 are fairly involved. I was wondering if we should given a high level view of our results here or earlier to help the readers?}\rnote{Again, I think that should be done in the intro.}
}

\newpage
\section{Discrepancy analysis and bounds}
\label{app:disc-bounds}

\DisYUpperBound*

\begin{proof}
For any hypothesis $h_0$ 
in $\sH$, we can write
\begin{align*}
 \dis_\sY(\h \sP, \sfq) & = \sup_{h \in \sH} \, 
\abs*{\E_{(x, y) \sim \h \sP}[\ell(h(x), y)]
- \sum_{i = 1}^m \sfq_i \ell(h(x_i), y_i)}\\
& \leq \sup_{h \in \sH} \, 
\abs*{\E_{x \sim \h \sP}[\ell(h(x), h_0(x))]
- \sum_{i = 1}^m \sfq_i \ell(h(x_i), h_0(x_i))}\\
& \quad + \sup_{h \in \sH} \,  \abs*{\E_{(x, y) \sim \h \sP}[\ell(h(x), y)]
- \E_{x \sim \h \sP}[\ell(h(x), h_0(x))]} \\
& \quad +
\sup_{h \in \sH} \, \abs*{\sum_{i = 1}^m \sfq_i \ell(h(x_i), h_0(x_i))
- \sum_{i = 1}^m \sfq_i \ell(h(x_i), y_i)} \\
& \leq \dis(\h \sP_\sX, \sfq) + \gamma
\E_{(x, y) \sim \h \sP} \bracket*{|y - h_0(x)|}
+ \gamma \sum_{i = 1}^m \sfq_i |y_i - h_0(x_i)|\\
& \leq \dis(\h \sP_\sX, \sfq) + \gamma \curl*{
\sup_{(x, y) \in \supp(\h \sP)} |y - h_0(x)|
+ \sup_{(x, y) \in \supp(\h \sQ)} |y - h_0(x)|}\\
& = \dis(\h \sP_\sX, \sfq) + \gamma \eta_\sH(\supp(\sP), \supp(\sQ)),
\end{align*}
which completes the proof.
\end{proof}

\wdis*

\begin{proof}
\begin{align*}
\dis(\h P, \sfq)
& = \max_{\norm{w}, \norm{w'} \leq \Lambda} 
\E_{x \sim \sfq} \abs*{\bracket*{\bracket*{(w - w') \cdot x}^2}
- 
\E_{x \sim \h \sP_X} \bracket*{\bracket*{(w - w') \cdot x}^2}}\\
& = \max_{\norm{w}, \norm{w'} \leq \Lambda} 
\abs*{\sum_{x \in \sX} \bracket*{\h P(x) - \sfq(x)} \bracket*{(w - w') \cdot x}^2}\\
& = \max_{\norm{u} \leq 2\Lambda} 
\abs*{\sum_{x \in \sX} \bracket*{\h P(x) - \sfq(x)} \bracket*{u \cdot x}^2}\\
& = \max_{\norm{u} \leq 2\Lambda} 
\abs*{u^\top \bracket*{\sum_{x \in \sX} \paren*{\h P(x) - \sfq(x)}
  xx^\top} u}\\
& = 4 \Lambda^2 \max_{\norm{u} \leq 1} 
\abs*{u^\top \bracket*{\bM_0 - \sum_{i = 1}^m \sfq_i \bM_i} u}\\
& = 4 \Lambda^2 \max_{\norm{u} \leq 1} \abs*{u^\top \bM(\sfq) u}\\
& = 4 \Lambda^2 \max_{\norm{u} = 1} \abs*{u^\top \bM(\sfq) u}\\
& = 4 \Lambda^2 \max\curl[\big]{\lambda_{\max} \paren*{\bM(\sfq)}, \lambda_{\max} \paren*{-\bM(\sfq)}}.
\end{align*}
This completes the proof.
\end{proof}

\newpage
\section{Smooth approximations}

\subsection{Softmax approximation}
\label{app:F}

\begin{restatable}{proposition}{FSmoothness}
\label{prop:F-smoothness}
Assume that $f$ is $\gamma$-smooth w.r.t. $\norm{\, \cdot \,}_2$, then $F$ is
$\gamma \paren*{\max_{i \in [m]}
  \norm{x_i}_2^4}$-smooth w.r.t. $\norm{\, \cdot \,}_1$.
\end{restatable}

\begin{proof}
For any $\sfq, \sfq' \in \Delta(m)$, the following upper bound on the spectral norm of $\bM(\sfq) - \bM(\sfq')$ holds: 
\begin{align}
\label{eq:ineq1}
\norm{\bM(\sfq) - \bM(\sfq')}_2 
& = \norm*{\sum_{i = 1}^m (\sfq_i - \sfq'_i) x_ix_i^\top}_2\\
& \leq \sum_{i = 1}^m \abs*{\sfq_i - \sfq'_i} \norm*{x_ix_i^\top}_2 \nonumber\\
& \leq \norm*{\sfq - \sfq'}_1 \max_{i \in [m]}\norm*{x_ix_i^\top}_2
\tag{H\"{o}lder's ineq.} \nonumber\\
& = \norm*{\sfq - \sfq'}_1 \max_{i \in [m]} \norm*{x_i}^2_2
\tag{$x_ix_i^\top$ admits a single non-zero eigenvalue, $\norm*{x_i}^2_2$}.\nonumber
\end{align}
We have $F(\sfq) = f(\bM(\sfq))$, thus the gradient of $F$ can be
expressed as follows: 
\[
\nabla F(\sfq)
= -[\tri{\nabla f(\bM(\sfq)), \bM_i}]_{i \in [m]}.
\]
Thus, for any $\sfq, \sfq' \in \Delta(m)$, we have:
\begin{align*}
\norm{\nabla F(\sfq) - \nabla F(\sfq')}_\infty
& = \max_{i \in [m]} \, \abs*{\tri{\nabla f(\bM(\sfq)) - \nabla f(\bM(\sfq')), \bM_i}}\\
& \leq \max_{i \in [m]} \, \norm*{\nabla f(\bM(\sfq)) - \nabla f(\bM(\sfq'))}_{(1)}
\norm*{\bM_i}_{(\infty)}
\tag{H\"{o}lder's ineq.}\\
& \leq \gamma \max_{i \in [m]} \, \norm*{\bM(\sfq) - \bM(\sfq')}_{(\infty)}
\norm*{\bM_i}_{(\infty)}
\tag{$\gamma$-smoothness of $f$}\\
& = \gamma \max_{i \in [m]} \, \norm*{\bM(\sfq) - \bM(\sfq')}_{2}
\norm*{\bM_i}_{2}
\tag{definition of $\norm{\, \cdot \,}_{(\infty)}$}\\
& \leq \gamma \max_{i \in [m]} \curl*{\norm*{\sfq - \sfq'}_1 \max_{i \in [m]} \norm*{x_i}^2_2} \norm*{x_i}^2_2
\tag{inequality \eqref{eq:ineq1}}\\
& = \gamma \paren*{\max_{i \in [m]} \norm*{x_i}^4_2} \norm*{\sfq - \sfq'}_1 .
\end{align*}
This completes the proof.
\end{proof}

We will use the following bound for the Hessian of $f$.
\begin{restatable}[\cite{Nesterov2007}]{lemma}{FHessian}
\label{lemma:f-hessian}
The following upper bound holds for the Hessian of $f$ for any two symmetric matrices $\bM, \bU \in \Sset_d$:
\[
\tri*{\nabla^2 \! f(\bM) \, \bU, \bU} \leq \mu \norm*{\bU}^2_2,
\]
where $\norm*{\bU}_2 = \norm{\lambda(\bU)}_\infty$ denotes the spectral norm of $\bU$.
\end{restatable}

\ThFSmoothness*

\begin{proof}
  In view of Lemma~\ref{lemma:f-hessian}, $f$ is $\norm{\, \cdot
    \,}_2$-$\mu$-smooth. The result thus follows by
  Proposition~\ref{prop:F-smoothness}.
\end{proof}

\ThFSensitivity*

\begin{proof}
  For $\bM(\sfq))$ and $\bM'(\sfq))$ differing only by point $x$ and
  $x'$ in $\h \sP_\sX$, we have:
\begin{align}
\label{eq:ineq3}
\norm{\bM(\sfq) - \bM'(\sfq)}_2 
& = \norm*{\frac{1}{n} \bracket*{xx^\top - x'x'^\top} }_2
\leq \frac{2r^2}{n}.
\end{align}
Thus, following the proof of Proposition~\ref{prop:F-smoothness},
the sensitivity is bounded by
\begin{align*}
\max_{i \in [m]} \, \abs*{\tri{\nabla f(\bM(\sfq)) - \nabla f(\bM'(\sfq)), \bM_i}}
& \leq \max_{i \in [m]} \, \norm*{\nabla f(\bM(\sfq)) - \nabla f(\bM'(\sfq))}_{(1)}
\norm*{\bM_i}_{(\infty)}
\tag{H\"{o}lder's ineq.}\\
& \leq \mu \max_{i \in [m]} \, \norm*{\bM(\sfq) - \bM(\sfq')}_{(\infty)}
\norm*{\bM_i}_{(\infty)}
\tag{$\mu$-smoothness of $f$}\\
& = \mu \max_{i \in [m]} \, \norm*{\bM(\sfq) - \bM(\sfq')}_{2}
\norm*{\bM_i}_{2}
\tag{definition of $\norm{\, \cdot \,}_{(\infty)}$}\\
& \leq \frac{2\mu r^2}{n}  \max_{i \in [m]} \norm*{x_i}^2_2.
\end{align*}
This completes the proof.
\end{proof}

\begin{proposition}
  \label{prop:F-hessian}
  The following inequality holds for the spectral norm of the Hessian of $F$:
\[
\norm{\nabla^2 F}_2
\leq \mu \norm*{\sum_{i = 1}^m x_ix_i^\top}_2 \leq \mu \bracket*{\sum_{i = 1}^m
  \norm{x_i}_2^2}.
\]  
\end{proposition}

\begin{proof}
The second-partial derivatives of $F(\sfq)$ can be expressed as follows:
\begin{align*}
\frac{\partial^2 S}{\partial \sfq_i \partial \sfq_j}
& = -\tri*{\frac{\partial}{\partial \sfq_j} \nabla f(\bM(\sfq)), \bM_i}\\
& = + \tri{\nabla^2 \! f(\bM(\sfq)) \, \bM_j, \bM_i}.
\end{align*}
Thus, using the shorthand $\ov \bM = \sum_{i = 1}^m \bX_i \bM_i$,
for any $\bX \in \Rset^m$, we can write:
\begin{align*}
  \bX^\top \nabla^2 F \bX
  & = \sum_{i, j = 1}^d  \bX_i \bX_j \tri*{\nabla^2 \! f(\bM(\sfq)) \, \bM_j, \bM_i}\\
  & = \tri*{\nabla^2 \! f(\bM(\sfq)) \, \paren[\bigg]{\sum_{j = 1}^d \bX_j \bM_j}, \paren*{\sum_{i = 1}^d \bX_i \bM_i}}\\
  & = \tri*{\nabla^2 \! f(\bM(\sfq)) \, \paren*{\ov \bM}, \paren*{\ov \bM}} \\
  & \leq \mu \norm{\ov \bM}^2_2
  \tag{Lemma~\ref{lemma:f-hessian}}\\
  & = \mu \paren*{\norm*{\sum_{i = 1}^m \bX_i x_ix_i^\top}_2}^2\\
  & = \mu \paren*{\max_{\norm{u} \leq 1} \abs*{\sum_{i = 1}^m \bX_i u^\top x_ix_i^\top u}}^2
  \tag{\text{def. of spectral norm}}\\
  & = \mu \paren*{\max_{\norm{u} \leq 1} \abs*{\sum_{i = 1}^m \bX_i (u^\top x_i)^2}}^2\\
  & \leq \mu \paren*{\max_{\norm{u} \leq 1} \norm{\bX} \sqrt{\sum_{i = 1}^m (u^\top x_i)^2}}^2
  \tag{\text{Cauchy-Schwarz ineq.}}\\
  & = \mu \paren*{\norm{\bX} \sqrt{\max_{\norm{u} \leq 1} \sum_{i = 1}^m (u^\top x_i)^2}}^2\\
  & = \mu \norm*{\sum_{i = 1}^m x_ix_i^\top}_2 \norm{\bX}^2.
\end{align*}
This completes the proof.
\end{proof}

\ignore{
We have
\begin{align*}
\tri*{\bM^{p}(\sfq) \bM_j \bM^{k - 2 - p}(\sfq), \bM_i}
& = \Tr\bracket*{\bM^{k - 2 - p}(\sfq) \bM_j \bM^{p}(\sfq)  \bM_i }\\
& = \Tr\bracket*{\bM^{k - 2 - p}(\sfq) x_j x_j^\top \bM^{p}(\sfq) x_i x_i^\top }\\
& = \paren*{x_j^\top \bM^{p}(\sfq) x_i} \Tr\bracket*{\bM^{k - 2 - p}(\sfq) x_j  x_i^\top }\\
& = \paren*{x_j^\top \bM^{p}(\sfq) x_i} \paren*{x_j^\top \bM^{k - 2 - p}(\sfq) x_i}.
\end{align*}
}

\ThSLipschitz*

\begin{proof}
  By inequality \eqref{eq:gradient_F}, for any $i \in [m]$,
  we have
  \begin{align*}
    \abs*{\bracket*{\nabla F(\sfq)}_i}
    & = \abs*{ \frac{\tri*{ \exp\paren*{\mu \bM(\sfq)}, \bM_i}}
      {\Tr\paren*{\exp\paren*{\mu \bM(\sfq)}}} }\\
    & = \frac{x_i^\top \exp\paren*{\mu \bM(\sfq)} x_i}
         {\Tr\paren*{\exp\paren*{\mu \bM(\sfq)}}}\\
    & \leq \norm{x_i}^2_2 
         \frac{\max_{\norm{u}_2 = 1} u^\top \exp\paren*{\mu \bM(\sfq)} u}
              {\Tr\paren*{\exp\paren*{\mu \bM(\sfq)}}}\\
    & = \norm{x_i}^2_2 \frac{\lambda_{\max} \paren*{\exp\paren*{\mu \bM(\sfq)}}}
              {\Tr\paren*{\exp\paren*{\mu \bM(\sfq)}}}
    \leq \norm{x_i}^2_2.
  \end{align*}
This completes the proof.
\end{proof}

\subsection{Properties of \texorpdfstring{$\wt F$}{F}}
\label{app:wtF}

\wtF*

\begin{proof}
The results follow directly the definition of $\wt F$ and
Theorems~\ref{th:S-smoothness}, \ref{th:S-sensitivity},
\ref{th:S-lipschitz} and the discussion above.
In particular, since $\wt F(\sfq) = f(\wt \bM(\sfq))$,
the gradient of $\wt F$ can be expressed
as follows in terms of $f$:
\[
\nabla \wt F(\sfq)
= - \tri*{\nabla f(\wt \bM(\sfq)), \diag(\bM_i, -\bM_i)}.
\]
Thus, for any $i \in [m]$, we have:
\begin{align*}
  \bracket*{\nabla \wt F(\sfq)}_i
  & = - \frac{\tri*{ \exp\paren*{\mu \wt \bM(\sfq)}, \diag(\bM_i, -\bM_i)}}
           {\Tr\paren*{\exp\paren*{\mu \wt \bM(\sfq)}}}.   
  \end{align*}
In particular, we can write:
\begin{align*}
  & \abs*{\bracket*{\nabla \wt F(\sfq)}_i}\\
  & = - \frac{x_i^\top \bracket*{\exp\paren*{\mu \bM(\sfq)} - \exp\paren*{-\mu \bM(\sfq)}} x_i}
           {\Tr\paren*{\exp\paren*{\mu \bM(\sfq)}} + \Tr\paren*{\exp\paren*{-\mu \bM(\sfq)}}}\\
  & \leq \norm{x_i}^2_2 \max_{\norm{u}^2_2 = 1}\abs*{ \frac{u^\top \bracket*{\exp\paren*{\mu \bM(\sfq)} - \exp\paren*{-\mu \bM(\sfq)}} u}
    {\Tr\paren*{\exp\paren*{\mu \bM(\sfq)}} + \Tr\paren*{\exp\paren*{-\mu \bM(\sfq)}}} }\\
  & \leq \norm{x_i}^2_2 \frac{\lambda_{\max} \paren*{\exp\paren*{\mu \bM(\sfq)}} +
      \lambda_{\max} \paren*{\exp\paren*{-\mu \bM(\sfq)}}}
    {\Tr\paren*{\exp\paren*{\mu \bM(\sfq)}} + \Tr\paren*{\exp\paren*{-\mu \bM(\sfq)}}} \\
  & \leq \norm{x_i}^2_2.
\end{align*}
This completes the proof.
\end{proof}
In the following, we further give explicit proofs of some of these
statements.


\begin{restatable}{proposition}{wtFSmoothness}
\label{prop:wtF-smoothness}
Assume that $f$ is $\gamma$-smooth w.r.t. $\norm{\, \cdot \,}_2$, then $\wt F$ is
$\gamma \paren*{\max_{i \in [m]}
  \norm{x_i}_2^4}$-smooth w.r.t. $\norm{\, \cdot \,}_1$.
\end{restatable}

\begin{proof}
For any $\sfq, \sfq' \in \Delta(m)$, the following upper bound on the spectral norm of $\bM(\sfq) - \bM(\sfq')$ holds: 
\begin{align}
\label{eq:wtineq1}
\norm{\wt \bM(\sfq) - \wt \bM(\sfq')}_2
& = \norm{\diag(\bM(\sfq) - \bM(\sfq'), -\bracket*{\diag(\bM(\sfq) - \bM(\sfq'))}}_2\\
& = \norm{\bM(\sfq) - \bM(\sfq')}_2\\
& = \norm*{\sum_{i = 1}^m (\sfq_i - \sfq'_i) x_ix_i^\top}_2\\
& \leq \sum_{i = 1}^m \abs*{\sfq_i - \sfq'_i} \norm*{x_ix_i^\top}_2 \nonumber\\
& \leq \norm*{\sfq - \sfq'}_1 \max_{i \in [m]}\norm*{x_ix_i^\top}_2
\tag{H\"{o}lder's ineq.} \nonumber\\
& = \norm*{\sfq - \sfq'}_1 \max_{i \in [m]} \norm*{x_i}^2_2
\tag{$x_ix_i^\top$ admits a single non-zero eigenvalue, $\norm*{x_i}^2_2$}.\nonumber
\end{align}
We have $F(\sfq) = f(\bM(\sfq))$, thus the gradient of $F$ can be
expressed as follows: 
\[
\nabla F(\sfq)
= -[\tri{\nabla f(\bM(\sfq)), \bM_i}]_{i \in [m]}.
\]
Thus, for any $\sfq, \sfq' \in \Delta(m)$, we have:
\begin{align*}
\norm{\nabla \wt F(\sfq) - \nabla \wt F(\sfq')}_\infty
& = \max_{i \in [m]} \, \abs*{\tri{\nabla f(\wt \bM(\sfq)) - \nabla f(\wt \bM(\sfq')), \diag(\bM_i, -\bM_i)}}\\
& \leq \max_{i \in [m]} \, \norm*{\nabla f(\wt \bM(\sfq)) - \nabla f(\wt \bM(\sfq'))}_{(1)}
\norm*{\diag(\bM_i, -\bM_i)}_{(\infty)}
\tag{H\"{o}lder's ineq.}\\
& \leq \gamma \max_{i \in [m]} \, \norm*{\wt \bM(\sfq) - \wt \bM(\sfq')}_{(\infty)}
\norm*{\bM_i}_{(\infty)}
\tag{$\gamma$-smoothness of $f$}\\
& = \gamma \max_{i \in [m]} \, \norm*{\wt \bM(\sfq) - \wt \bM(\sfq')}_{2}
\norm*{\bM_i}_{2}
\tag{definition of $\norm{\, \cdot \,}_{(\infty)}$}\\
& \leq \gamma \max_{i \in [m]} \curl*{\norm*{\sfq - \sfq'}_1 \max_{i \in [m]} \norm*{x_i}^2_2} \norm*{x_i}^2_2
\tag{inequality \eqref{eq:wtineq1}}\\
& = \gamma \paren*{\max_{i \in [m]} \norm*{x_i}^4_2} \norm*{\sfq - \sfq'}_1 .
\end{align*}
This completes the proof.
\end{proof}


\begin{proof}
  For $\bM(\sfq))$ and $\bM'(\sfq))$ differing only by point $x$ and
  $x'$ in $\h \sP_\sX$, we have:
\begin{align}
\label{eq:ineqwt3}
\norm{\wt \bM(\sfq) - \wt \bM'(\sfq)}_2 
& = \norm{\bM(\sfq) - \bM'(\sfq)}_2 \\
& = \norm*{\frac{1}{n} \bracket*{xx^\top - x'x'^\top} }_2
\leq \frac{2r^2}{n}.
\end{align}
Thus, following the proof of Proposition~\ref{prop:wtF-smoothness},
the sensitivity is bounded by
\begin{align*}
& \max_{i \in [m]} \, \abs*{\tri{\nabla f(\wt \bM(\sfq)) - \nabla f(\wt \bM'(\sfq)),
      \diag(\bM_i, -\bM_i)}}\\
& \leq \max_{i \in [m]} \, \norm*{\nabla f(\wt \bM(\sfq)) - \nabla f(\wt \bM'(\sfq))}_{(1)}
\norm*{\diag(\bM_i, -\bM_i)}_{(\infty)}
\tag{H\"{o}lder's ineq.}\\
& \leq \mu \max_{i \in [m]} \, \norm*{\wt \bM(\sfq) - \wt \bM(\sfq')}_{(\infty)}
\norm*{\bM_i}_{(\infty)}
\tag{$\mu$-smoothness of $f$}\\
& = \mu \max_{i \in [m]} \, \norm*{\bM(\sfq) - \bM(\sfq')}_{2}
\norm*{\bM_i}_{2}
\tag{definition of $\norm{\, \cdot \,}_{(\infty)}$}\\
& \leq \frac{2\mu r^2}{n}  \max_{i \in [m]} \norm*{x_i}^2_2.
\end{align*}
This completes the proof.
\end{proof}

\subsection{\texorpdfstring{$p$}{p}-norm approximation}
\label{app:G}

A smooth approximation of
$\norm*{\bM(\sfq)}^2_2
= \norm*{\lambda\paren[\big]{\bM(\sfq)}}_\infty^2$ can be defined as
follows:
\[
G(\sfq) 
= \Tr\bracket*{\bM(\sfq)^{2p}}^{\frac{1}{p}}
= \bracket*{\sum_{i = 1}^d \lambda_i(\bM(\sfq))^{2p}}^{\frac{1}{p}},
\]
for $p$ sufficiently large. 
The following inequalities hold for this
approximation:
\[
\norm*{\lambda\paren[\big]{\bM(\sfq)}}_\infty^2 \leq G(\sfq) \leq 
\bracket*{\rank(\bM(\sfq))}^{\frac{1}{p}} \norm*{\lambda\paren[\big]{\bM(\sfq)}}_\infty^2.
\]
The gradient of the smooth approximation is given for all $i \in [1, m]$
by:
\begin{equation}
\label{eq:gradient}
\mspace{-3mu}
[\nabla G(\bM(\sfq))]_i 
= -2 \tri*{\bM^{2p - 1}(\sfq), \bM_i} \Tr[\bM^{2p}(\sfq)]^{\frac{1}{p} - 1}.
\mspace{-12mu}
\end{equation}
We can write $G(\sfq) = g(\bM(\sfq))$ where $g$ is defined for all $\bM \in \Sset_d$ by
\[
g(\bM)
= \Tr\bracket*{\bM^{2p}}^{\frac{1}{p}}
= \tri*{\bM^{2p}, \bI}^{\frac{1}{p}}.
\]
The following result provides the desired smoothness result needed
for $G$, which we prove by using the smoothness property of $g$.

\begin{restatable}{theorem}{ThGSmoothness}
\label{th:G-smoothness}
The $p$-norm approximation function $G$ is $(2p - 1) \paren*{\max_{i
    \in [m]} \norm{x_i}_2^4}$-smooth for $\norm{\, \cdot \,}_1$.
\end{restatable}
The proof is given in Appendix~\ref{app:G}. Next, we present
a sensitivity bounds for $G$.

\begin{restatable}{theorem}{ThGSensitivity}
  \label{th:p-sensitivity}
  Assume that the support of $\sP$ is included in the $\ell_2$ ball of
  radius $r$. Then, the gradient of the $p$-norm approximation $G$ is
  $\frac{2(2p - 1)r^2}{n} \max_{i \in [m]} \norm*{x_i}^2_2$-sensitive.
\end{restatable}
The proof is given in Appendix~\ref{app:G}.

\begin{proposition}
\label{prop:G-smoothness}
Assume that $g$ is $\gamma$-smooth with respect to the norm $\norm{\,
  \cdot \,}_{(2p)}$:
\[
\forall \bM, \bM' \in \Sset_d, \quad
\norm{\nabla g(\bM) - \nabla g(\bM')}_{(r)}
\leq \gamma \norm{\nabla g(\bM) - \nabla g(\bM')}_{(2p)},
\]
with $\frac{1}{r} + \frac{1}{2p} = 1$.
Then, $G$ is $\gamma \paren*{\max_{i \in [m]} \norm{x_i}_2^4}$-smooth:
\[
\forall \sfq, \sfq' \in \Rset^d, \quad
\norm{\nabla G(\sfq) - \nabla G(\sfq')}_{\infty}
\leq \gamma \paren*{\max_{i \in [m]} \norm*{x_i}^4_2} \norm{\sfq - \sfq'}_1 .
\]

\end{proposition}

\begin{proof}
  The proof is similar to that of Proposition~\ref{prop:F-smoothness}.
  For any $\sfq, \sfq' \in \Delta(m)$, the following upper bound on
  the norm-$(2p)$ of $\bM(\sfq) - \bM(\sfq')$ holds:
\begin{align}
\label{eq:ineq2}
\norm{\bM(\sfq) - \bM(\sfq')}_{(2p)}
& = \norm*{\sum_{i = 1}^m (\sfq_i - \sfq'_i) x_ix_i^\top}_{(2p)}\\
& \leq \sum_{i = 1}^m \abs*{\sfq_i - \sfq'_i} \norm*{x_ix_i^\top}_{(2p)} \nonumber\\
& \leq \norm*{\sfq - \sfq'}_1 \max_{i \in [m]}\norm*{x_ix_i^\top}_{(2p)}
\tag{H\"{o}lder's ineq.} \nonumber\\
& = \norm*{\sfq - \sfq'}_1 \max_{i \in [m]} \norm*{x_i}^2_2
\tag{$x_ix_i^\top$ admits a single non-zero eigenvalue, $\norm*{x_i}^2_2$}.\nonumber
\end{align}
We have $G(\sfq) = g(\bM(\sfq))$, thus the gradient of $G$ can be
expressed as follows:
\[
\nabla G(\sfq)
= -[\tri{\nabla g(\bM(\sfq)), \bM_i}]_{i \in [m]}.
\]
Thus, for any $\sfq, \sfq' \in \Delta(m)$, we have:
\begin{align*}
\norm{\nabla G(\bM(\sfq)) - \nabla G(\bM(\sfq'))}_\infty
& = \max_{i \in [m]} \, \abs*{\tri{\nabla g(\bM(\sfq)) - \nabla g(\bM(\sfq')), \bM_i}}\\
& \leq \max_{i \in [m]} \, \norm*{\nabla g(\bM(\sfq)) - \nabla g(\bM(\sfq'))}_{(r)}
\norm*{\bM_i}_{(2p)}
\tag{H\"{o}lder's ineq.}\\
& \leq \gamma \max_{i \in [m]} \, \norm*{\bM(\sfq) - \bM(\sfq')}_{(2p)}
\norm*{\bM_i}_{(2p)}
\tag{$\gamma$-smoothness of $f$}\\
& \leq \gamma \max_{i \in [m]} \curl*{\norm*{\sfq - \sfq'}_1 \max_{i \in [m]} \norm*{x_i}^2_2} \norm*{x_i}^2_2
\tag{inequality \eqref{eq:ineq2}}\\
& = \gamma \paren*{\max_{i \in [m]} \norm*{x_i}^4_2} \norm*{\sfq - \sfq'}_1 .
\end{align*}
This completes the proof.
\end{proof}

We will use the following bound for the Hessian of $g$.
\begin{lemma}[\cite{Nesterov2007}]
  \label{lemma:g-hessian}
The following upper bound holds for the Hessian of $f$ for any two symmetric matrices $\bM, \bU \in \Sset_d$:
\[
\tri*{\nabla^2 \! g(\bM) \, \bU, \bU} \leq (2p - 1) \norm*{\lambda(\bU)}^2_{2p},
\]
where $\norm*{\lambda(\bU)}^2_{2p} = \paren*{\Tr\bracket*{\bU^{2p}}}^{\frac{1}{p}}$.
\end{lemma}

\ThGSmoothness*

\begin{proof}
  In view of Lemma~\ref{lemma:g-hessian}, $g$ is $\norm{\, \cdot
    \,}_{(2p)}$-$(2p - 1)$-smooth. The result thus follows by
  Proposition~\ref{prop:G-smoothness}.
\end{proof}

\ThGSensitivity*

\begin{proof}
  For $\bM(\sfq))$ and $\bM'(\sfq))$ differing only by point $x$ and
  $x'$ in $\h \sP_\sX$, we have:
\begin{align}
\label{eq:ineq4}
\norm{\bM(\sfq) - \bM'(\sfq)}_2 
& = \norm*{\frac{1}{n} \bracket*{xx^\top - x'x'^\top} }_2
\leq \frac{2r^2}{n}.
\end{align}
Thus, following the proof of Proposition~\ref{prop:G-smoothness},
the sensitivity is bounded by
\begin{align*}
\max_{i \in [m]} \, \abs*{\tri{\nabla g(\bM(\sfq)) - \nabla g(\bM'(\sfq)), \bM_i}}
& \leq \max_{i \in [m]} \, \norm*{\nabla g(\bM(\sfq)) - \nabla g(\bM'(\sfq))}_{(r)}
\norm*{\bM_i}_{(2p)}
\tag{H\"{o}lder's ineq.}\\
& \leq (2p - 1) \max_{i \in [m]} \, \norm*{\bM(\sfq) - \bM(\sfq')}_{(2p)}
\norm*{\bM_i}_{(2p)}
\tag{$\gamma$-smoothness of $f$}\\
& \leq  \frac{2(2p - 1)r^2}{n}
\max_{i \in [m]} \norm*{x_i}^2_2.
\tag{inequality \eqref{eq:ineq4}}
\end{align*}
This completes the proof.
\end{proof}

\begin{proposition}
  \label{prop:G-hessian}
  The following inequality holds for the spectral norm of the Hessian of $F$:
\[
\norm{\nabla^2 G}_2
\leq (2p - 1) \bracket*{\sum_{i = 1}^m  \norm{x_i}_2^2}.
\]  
\end{proposition}

\begin{proof}
As in the proof of Proposition~\ref{prop:F-hessian}, we have:
\begin{align*}
\frac{\partial^2 G}{\partial \sfq_i \partial \sfq_j}
& = -\tri*{\frac{\partial}{\partial \sfq_j} \nabla f(\bM(\sfq)), \bM_i}\\
& = + \tri{\nabla^2 \! g(\bM(\sfq)) \, \bM_j, \bM_i}.
\end{align*}
Thus, using the shorthand $\ov \bM = \sum_{i = 1}^m \bX_i \bM_i$,
for any $\bX \in \Rset^m$, we can write:
\begin{align*}
  \bX^\top \nabla^2 G \bX
  & = \sum_{i, j = 1}^d  \bX_i \bX_j \tri*{\nabla^2 \! g(\bM(\sfq)) \, \bM_j, \bM_i}\\
  & = \tri*{\nabla^2 \! f(\bM(\sfq)) \, \paren[\bigg]{\sum_{j = 1}^d \bX_j \bM_j}, \paren*{\sum_{i = 1}^d \bX_i \bM_i}}\\
  & = \tri*{\nabla^2 \! f(\bM(\sfq)) \, \paren*{\ov \bM}, \paren*{\ov \bM}} \\
  & \leq (2p - 1) \norm{\ov \bM}^2_{(2p)}
  \tag{Lemma~\ref{lemma:g-hessian}}\\
  & = (2p - 1) \paren*{\norm*{\sum_{i = 1}^m \bX_i x_ix_i^\top}_{(2p)}}^2\\
  & = (2p - 1) \paren*{\Tr \paren*{\bracket*{\sum_{i = 1}^m \bX_i x_ix_i^\top}^{2p}}}^{\frac{1}{p}}\\
  & \leq (2p - 1) \paren*{\sum_{i = 1}^m |\bX_i| \norm*{ x_ix_i^\top}_{(2p)}}^2\\
  & \leq (2p - 1) \norm{\bX}_2^2  \sum_{i = 1}^m \norm*{ x_ix_i^\top}^2_{(2p)}\\
  & = (2p - 1) \norm{\bX}_2^2  \sum_{i = 1}^m \norm*{x_i}^2_{2}.
\end{align*}
This completes the proof.

\end{proof}

\newpage
\section{Private Two-Stage Algorithms for Discrepancy Minimization}\label{app:two-stage}
\subsection{The Noisy Frank-Wolfe Algorithm and its Guarantees}\label{app:PrivateFW}

We first start with our noisy Frank-Wolfe algorithm. The algorithm is
formally described in Algorithm~\ref{Alg:nfw} below.

\begin{algorithm}[ht!]
    \caption{Noisy Frank-Wolfe for minimizing (regularized) smoothed discrepancy}
    \begin{algorithmic}[1]
    
    \REQUIRE Private unlabeled dataset $T=(\tx_1,\ldots, \tx_n)\in \cX^n$, public unlabeled dataset $S_\cX=(x_1, \ldots, x_m)\in\cX^m$,  privacy parameters $(\varepsilon, \delta)$,
    smooth-approximation parameter $\mu$, regularization parameter $\lambda$, $\#$ of iterations $K$.
    \STATE Let $r=\max_{x\in \cX}\norm{x}_2$.
    \STATE Let $\hat{r} = \max_{i\in [m]}\norm{x_i}_2$.
    \STATE Let $\Delta_m$ be the $(m-1)$-dimensional probability simplex.
    \STATE Define $\Freg(\sfq)\triangleq \wt F_T(\sfq)+\frac{\lambda}{2}\norm{\sfq}_2^2, ~\sfq \in \Delta_m.$
    \STATE Choose an arbitrary point $\sfq_1\in\Delta_m$.
     \STATE Set $\sigma = \frac{4\mu r^2 \hat{r}^2 \sqrt{2K\log(\frac{1}{\delta})}}{{n\varepsilon}}.$ 
    \FOR{$k=1$ to $K$}
    \STATE Compute $\nabla \Freg(\sfq_k)=\nabla \wt F_T(\sfq_k)+\lambda \sfq_k,$ where $\nabla \wt F_T(\sfq_k)$ is computed as described in Section~\ref{sec:soft-max}.
    \STATE Draw $\{b_{i, k}\}_{i\in [m]}$ i.i.d.\ from $\lap(\sigma).$ 
    \STATE Find $j_k = \argmin\limits_{i\in [m]} \{\langle \be_i, \nabla \Freg(\sfq_k) \rangle + b_{i, k}\},$ where $\{\be_i\}_{i\in [m]}$ are the standard unit vectors in $\Rset^m.$ 
    \STATE Update $\sfq_{k+1} = (1-\eta_k)\sfq_k + \eta_k \be_{j_k}$, where $\eta_k=\frac{3}{k+2}.$
    \ENDFOR
    \RETURN $\widehat{\sfq}=\sfq_K.$
   \end{algorithmic}
    \label{Alg:nfw}
\end{algorithm}
\vspace{-1ex}

\FWmain*

The above theorem follows as a corollary of the following theorem.
\begin{restatable}{theorem}{PrivateFW}
\label{th:PrivateFW}
Algorithm~\ref{Alg:nfw} is $(\varepsilon, \delta)$-differentially private. Let $\beta \in (0, 1)$. With probability $1-\beta$ over the algorithm's randomness (the Laplace noise), the output $\widehat{\sfq}$ satisfies 
\begin{align*}
  \Freg(\h \sfq)
  \leq& \min_{\sfq\in \Delta_m} \Freg(\sfq)
  +\frac{2(\mu \hat{r}^4+\lambda)}{K}
  +\frac{8\mu r^2 \hat{r}^2\sqrt{2K\log(\frac{1}{\delta})}\log(K)\log(\frac{mK}{\beta})}{\varepsilon n}.
\end{align*}
\end{restatable}
The proof relies on the smoothness property of $\Freg$ and the
sensitivity bound on $\nabla \wt F_T(\sfq)$. Using the approximation guarantee of
$\wt F_T$ given in Corollary~\ref{cor:wtf} together with
Theorem~\ref{th:PrivateFW} above, we reach the result of Theorem~\ref{th:FWmain}, which can be more precisely stated as the following corollary.
\begin{restatable}{corollary}{DiscApprox}
\label{th:DiscApprox}
Let $\sfq^\ast\in\argmin\limits_{\sfq\in \Delta_m}\dis(\h P,
\sfq)$. Let $\beta \in (0, 1)$. There exists a choice of $K$ and $\mu$ in Algorithm~\ref{Alg:nfw} for which the following holds: 
assuming w.l.o.g. that $\lambda\leq \mu \hat{r}^4$, w.p. at least $1 - \beta$, the output $\widehat{\sfq}$ satisfies
\begin{align*}
  \dis(\h P, \widehat{\sfq})
  \leq& \dis(\h P, \sfq^\ast) + \frac{\lambda}{2} \norm{\sfq^\ast}_2^2+ \wt{O}\paren*{\frac{\Lambda^4\hat{r}^{4/3}r^{2/3}}{(\varepsilon n)^{1/3}}},
\end{align*}
  where $\Lambda$ is the $\norm{\cdot}_2$-bound on the  predictors in $\sH$.
\end{restatable}

\textbf{Proof of Theorem~\ref{th:PrivateFW}}
For the privacy guarantee of Algorithm~\ref{Alg:nfw}, first note that the global $\norm{\cdot}_\infty$-sensitivity of $\nabla \Freg$ (w.r.t. replacing any data point in the private dataset) is the same as that of $\nabla \wt F_T$, which is bounded by $\frac{2\mu r^2\hat{r}^2}{n}$ as established in Corollary~\ref{cor:wtf} (Part~3). Hence, by the setting of the scale of the Laplace noise and the privacy guarantee of the Report-Noisy-Max mechanism \cite{DR14,bhaskar2010discovering}, it follows that a single iteration of Algorithm~\ref{Alg:nfw} is $\paren*{\frac{\varepsilon}{\sqrt{8 K\log(\frac{1}{\delta})}}, 0}$-differentially private. The advanced composition theorem of differential privacy \cite{DR14} thus implies that the algorithm is $(\varepsilon, \delta)$-differentially private. 

We now prove the convergence guarantee. Let $\tilde{\sfq} \in \argmin\limits_{\sfq\in \Delta_m}S_\lambda(\sfq)$. First, by Corollary~\ref{cor:wtf} (Part~2), $\wt F_T$ is $\mu \hat{r}^4$-smooth w.r.t. $\norm{\cdot}_1$. Note also that $\frac{\lambda}{2}\norm{\sfq}_2^2$ is $\lambda$-smooth over $\sfq\in\Delta_m$ w.r.t. $\norm{\cdot}_1$. This follows from the fact that for any $\sfq, \sfq'\in \Delta_m,$ $$\norm{\nabla\big(\frac{\lambda}{2}\norm{\sfq}_2^2\big)-\nabla\big(\frac{\lambda}{2}\norm{\sfq'}_2^2\big)}_\infty=\lambda\norm{\sfq-\sfq'}_\infty\leq \lambda\norm{\sfq-\sfq'}_1.$$ Hence, we get that the objective $\Freg$ is $(\mu \hat{r}^4+\lambda)$-smooth w.r.t. $\norm{\cdot}_1$ over $\Delta_m$. Thus, by standard analysis of the Noisy Frank-Wolfe algorithm (see, e.g., \cite{talwar2015nearly, BassilyGuzmanMenart2022}), we have
\begin{align*}
\Freg(\h{\sfq}) - \Freg(\tilde{\sfq}) &\leq \frac{2(\mu\hat{r}^4+\lambda)}{K}  + \sum_{k=1}^K \eta_k\alpha_k,
\end{align*} 
where $\alpha_k\triangleq \langle \nabla \Freg(\sfq_k), \be_{j_k}\rangle-\min\limits_{i\in [m]}\langle \nabla \Freg(\sfq_k), \be_i\rangle$. By the tail properties of the Laplace distribution together with the union bound, we get that w.p. $\geq 1-\beta$, for all $k\in [K],$ $\alpha_k\leq \sigma \log(Km/\beta)=\frac{4\mu r^2 \hat{r}^2 \sqrt{2K\log(\frac{1}{\delta})}\log(Km/\beta)}{{n\varepsilon}}$. Hence, given the setting of $\eta_k$, w.p. $\geq 1-\beta,$ the above bound simplifies to 
\begin{align*}
\Freg(\h{\sfq}) - \Freg(\tilde{\sfq}) &\leq \frac{2(\mu\hat{r}^4+\lambda)}{K}  + \frac{8\mu r^2 \hat{r}^2 \sqrt{2K\log(\frac{1}{\delta})}\log(K)\log(Km/\beta)}{{n\varepsilon}},
\end{align*} 
which completes the proof. 

\textbf{Proof of Corollary~\ref{th:DiscApprox}}.
 The result can be obtained with the
following choices of $K$ and $\mu$:
\begin{align*}
K &= \frac{\hat{r}^{4/3}(\varepsilon n)^{2/3}}{3 r^{4/3}\log^{1/3}(\frac{1}{\delta})\log^{2/3}(n) \log^{2/3}(\frac{mn}{\beta})}
\qquad \mu = \sqrt{\frac{K\log(m+n)}{8\hat{r}^4}}.
\end{align*}

\subsection{The Noisy Mirror-Descent Algorithm and its Guarantees}\label{app:PrivateMD}

Next, we give an alternative private algorithm for minimizing the
regularized smooth approximation of the discrepancy, $\Freg$. Our
algorithm is described in Algorithm~\ref{Alg:nmd} below. Compared to
the guarantees of the private Frank-Wolfe algorithm, the optimization error of this algorithm exhibits a better
dependence on $n$ at the expense of worse dependence on $m$. In
particular, the excess error with respect to the minimum discrepancy
scales as \ignore{$\widetilde{O}\paren*{\frac{m^{1/4}}{\sqrt{n}}}$}
$\widetilde{O}\paren*{\frac{m^{1/4}}{\sqrt{n}}}$ (see Corollary~\ref{th:DiscApprox2}). \ignore{Depending on the
regime of $m$ with respect to $n$, one algorithm would yield better
guarantees than the other. Specifically, when} When
$m=\widetilde{O}\paren*{n^{2/3}}$, Algorithm~\ref{Alg:nmd} below
benefits from more favorable generalization error guarantees than
Algorithm~\ref{Alg:nfw}. \ignore{The opposite holds when
$m=\widetilde{\Omega}\paren*{n^{2/3}}$.} 

\begin{algorithm}[ht!]
    \caption{Noisy Mirror-Descent for minimizing $\Freg$}
    \begin{algorithmic}[1]
    
    \REQUIRE Private unlabeled dataset $T=(\tx_1,\ldots, \tx_n)\in \cX^n$, public unlabeled dataset $S_\cX=(x_1, \ldots, x_m)\in\cX^m$,  privacy parameters $(\varepsilon, \delta)$,
    smooth-approximation parameter $\mu$, number of iterations $K$.
    \STATE Let $r=\max_{x\in \cX}\norm{x}_2$.
    \STATE Let $\hat{r} = \max_{i\in [m]}\norm{x_i}_2$.
    \STATE Let $\Delta_m$ be the $(m-1)$-dimensional probability simplex. 
    \STATE Let $p = 1+ \frac{1}{\log(m)}$.
     \STATE Set $\sigma = \frac{4\mu r^2 \hat{r}^2 \sqrt{2Km\log(\frac{1}{\delta})}}{{n\varepsilon}}$\label{step:MD-sigma}
     \STATE Set $\eta = \frac{2}{(\hat{r}^2+\lambda)}\sqrt{\frac{\log(m)}{K}}.$
    \STATE Choose an arbitrary point $\sfq_1\in\Delta_m$.
    \FOR{$k=1$ to $K$}
    \STATE Compute $\widehat{\nabla}_k=\nabla \wt F_T(\sfq_k)+\lambda \sfq_k + Z_k,$ where $Z_k\sim \cN(\mathbf{0}, \sigma^2\mathbb{I}_m)$.
    \STATE Update $\sfq_{k+1} = \arg\min\limits_{\sfq\in \Delta_m} \bigg\{\langle \widehat{\nabla}_k, \sfq-\sfq_k \rangle + \frac{\norm{\sfq-\sfq_k}_p^2}{\eta (p-1)}\bigg\}.$ 
    \ENDFOR
    \RETURN $\widehat{\sfq}=\frac{1}{K}\sum_{k=1}^K\sfq_k$
   \end{algorithmic}
    \label{Alg:nmd}
\end{algorithm}

\MDmain*

\noindent The above theorem follows as a corollary of the following theorem.
\begin{restatable}{theorem}{PrivateMD}
\label{th:PrivateMD}
Algorithm~\ref{Alg:nmd} is $(\varepsilon, \delta)$-differentially private. Let $\beta \in (0, 1)$. If we set 
\[
K = \frac{(\hat{r}^2+\lambda)^2\varepsilon^2 n^2}{128
  \mu^2\hat{r}^4r^4 m \log(\frac{2m}{\beta})\log(\frac{1}{\delta})}.
\]
then with probability at least $1 - \beta$ over the algorithm's randomness (Gaussian noise), the output
$\widehat{\sfq}$ satisfies
\begin{align*}
 \Freg(\h \sfq)
  &\leq \min_{\sfq\in \Delta_m} \Freg(\sfq)
  +\frac{46(\lambda+\hat{r}^2)\mu r^2 \hat{r}^2\log(\frac{2m}{\beta})\sqrt{m\log(\frac{1}{\delta})}}{\varepsilon n}.
\end{align*}
\end{restatable}
Using the approximation guarantee of $\wt F_T$ given in
Corollary~\ref{cor:wtf} together with
Theorem~\ref{th:PrivateMD} above, we reach the result of Theorem~\ref{th:MDmain}, which can be more precisely stated as the following corollary.
\begin{restatable}{corollary}{DiscApprox2}
\label{th:DiscApprox2}
Let $\sfq^\ast\in\arg\min\limits_{\sfq\in \Delta_m}\dis(\h P,
\sfq)$. Let $\beta\in (0, 1)$. In Algorithm~\ref{Alg:nmd}, set $K$ as in Theorem~\ref{th:PrivateMD}. Then, there exists a
choice of $\mu$ such that the following holds:
assuming w.l.o.g. that $\lambda = O(\hat{r}^2)$, w.p. at least $1 - \beta$, the output $\widehat{\sfq}$ satisfies 
\begin{align*}
  \dis(\h P, \widehat{\sfq})
  &\leq \dis(\h P, \sfq^\ast)) + \frac{\lambda}{2} \norm{\sfq^\ast}_2^2
  + \wt{O}\paren*{\frac{\Lambda^2 r \hat{r}^2m^{1/4}}{\sqrt{\varepsilon n}}},
\end{align*}
where $\Lambda$ is the $\norm{\cdot}_2$-bound on the  predictors in $\sH$.
\end{restatable}

\textbf{Proof of Theorem~\ref{th:PrivateMD}}  
First, we show that Algorithm~\ref{Alg:nmd} is $(\varepsilon, \delta)$-differentially private. Note that for any $\sfq\in \Delta_m$ the $\norm{\cdot}_2$-sensitivity of $\nabla \Freg$ can be upper bounded as $\norm{\nabla \Freg(\sfq)- \nabla\Fregprime(\sfq)}_2=\norm{\nabla \wt F_T(\sfq) - \nabla \wt F_{T'}(\sfq)}_2\leq \sqrt{m}\norm{\nabla \wt F_T(\sfq) - \nabla \wt F_{T'}(\sfq)}_{\infty}\leq \frac{2\mu r^2\hat{r}^2\sqrt{m}}{n}$, where the last inequality follows from the sensitivity bound in Corollary~\ref{cor:wtf}. Thus, given the setting of the Gaussian noise in the algorithm, the privacy guarantee of the Gaussian mechanism \cite{dwork2006our, DR14} together with the Moments Accountant technique \cite{abadi2016deep} show the claimed privacy guarantee. 

Next, we prove the convergence guarantee. The analysis here is similar to the analysis of noisy mirror descent in \cite{bassily2021non, AsiFeldmanKorenTalwar2021}. First, it is known that $\Phi(\sfq)\triangleq \frac{\norm{\sfq}^2_p}{p-1}$, where $p=1+\frac{1}{\log(m)},$ is $1$-strongly convex w.r.t. $\norm{\cdot}_1$ (see, e.g., \cite{NY82}). Moreover, $D_{\Phi}\triangleq \max\limits_{\sfq, \sfq'} \lvert \Phi(\sfq)- \Phi(\sfq')\rvert\leq 2\log(m)$. Note also that $\Freg$ is $\gamma \triangleq (\hat{r}^2+\lambda)$-Lipschitz w.r.t $\norm{\cdot}_1$, which follows from the Lipschitz property of $\wt F_T$ (Corollary~\ref{cor:wtf}) and the fact that $\frac{\lambda}{2}\norm{\sfq}_2^2$ is $\lambda$-Lipschitz w.r.t. $\norm{\cdot}_1$ over $\Delta_m$. Hence, by standard analysis of (noisy) mirror descent \cite{NY82,nemirovski2009robust}, we have (letting $\tilde{\sfq}=\argmin\limits_{\sfq\in \Delta_m}\Freg(\sfq)$)
\begin{align*}
    \Freg(\h{\sfq})-\Freg(\tilde{\sfq}) \leq & \frac{D_{\Phi}}{2\eta K} + \frac{\eta \gamma^2}{2} + \frac{\eta}{2 K}\sum_{k=1}^K\norm{Z_k}^2_\infty \\
    \leq & \frac{2\log(m)}{\eta K} + \frac{\eta (\lambda +\hat{r}^2)^2}{2} + \frac{\eta}{2 K}\sum_{k=1}^K\norm{Z_k}^2_\infty
\end{align*}
where $\{Z_k: k\in [K]\}$ are i.i.d.\ from $\cN(\mathbf{0}, \sigma^2 \mathbb{I}_m)$. By a concentration argument in non-Euclidean norms \cite[Theorem 2.1]{Juditsky:2008},
w.p. $\geq 1-\beta$, we have $\frac{1}{K}\sum_{k=1}^K\norm{Z_k}^2_\infty \leq 4\sigma^2 \log(\frac{2m}{\beta})$. Hence, w.p. $\geq 1-\beta$, we have 
\begin{align*}
    \Freg(\h{\sfq})-\Freg(\tilde{\sfq}) \leq & \frac{2\log(m)}{\eta K} + \frac{\eta (\lambda +\hat{r}^2)^2}{2} + 2\eta\sigma^2 \log(\frac{2m}{\beta}) 
\end{align*}
Thus, given the setting of $\sigma$ (Step~\ref{step:MD-sigma} of Algorithm~\ref{Alg:nmd}), optimizing the bound above in $\eta$ and $K$ yields $\eta = \frac{2}{(\hat{r}^2+\lambda)}\sqrt{\frac{\log(m)}{K}}$ and $K= \frac{(\hat{r}^2+\lambda)^2\varepsilon^2 n^2}{128\mu^2\hat{r}^4r^4 m \log(\frac{2m}{\beta})\log(\frac{1}{\delta})}$. Plugging these values in the above bound yields the claimed bound.

\textbf{Proof of Corollary~\ref{th:DiscApprox2}}. The following is the choice of $\mu$ yielding the statement of the corollary:
\[
\mu = \frac{\sqrt{\varepsilon n} \log^{1/4}(m+n)}{4r\hat{r}\sqrt{(\lambda+\hat{r}^2)\log(\frac{2m}{\beta})}\bracket*{m\log(\frac{1}{\delta})}^{1/4}}.
\]

\section{Proof of Theorem~\ref{th:stFW}}
\label{app:stFW}

\stFW*

The statement holds with the following choice of
$K$ and $\mu$:
\begin{align*}
    K&=\frac{\sqrt{2}\Bar{D}}{\sigma^0_\sfq \log\paren*{\frac{\Bar{D}J}{\sigma^0_\sfq \beta}}+D_w\sigma^0_w\sqrt{d\log\paren*{\frac{\Bar{D}}{D_w\sigma^0_w \beta}}}} 
    \qquad \eta= \sqrt{\frac{2(D_\sfq \gamma_\sfq+D_w\gamma_w)}{(D_\sfq^2\mu_\sfq+D_w^2\mu_w+2\gamma_{\sfq, w}D_\sfq D_w)K}},
\end{align*}
where 
$\Bar D = \sqrt{(D_\sfq \gamma_\sfq+D_w\gamma_w)(D_\sfq^2\mu_\sfq+D_w^2\mu_w+2\gamma_{\sfq, w}D_\sfq D_w)}$, $\sigma^0_\sfq=\frac{\sigma_\sfq}{\sqrt{K}}$, and $\sigma^0_w=\frac{\sigma_w}{\sqrt{K}}$ (where $\sigma_\sfq, \sigma_w$ are as given in steps~\ref{step:sigma_q} and \ref{step:sigma_w}).

\begin{proof}
The privacy proof follows by combining the guarantees of the Report-Noisy-Min mechanism (steps~\ref{step:sigma_q}, \ref{step:stFW-report-noisy-min1}, and \ref{step:stFW-report-noisy-min2}) and the Gaussian mechanism (steps~\ref{step:sigma_w} and \ref{step:gauss-mech}) together with the application of the advanced composition theorem of differential privacy over the $K$ rounds of the algorithm. 

We now prove the convergence (stationarity gap) guarantee. By the smoothness of $f_T$, we have
\begin{align*}
    & f_T(\sfq^{k+1}, w^{k+1})\\
    & \leq f_T(\sfq^{k}, w^{k+1})+\langle \nabla_\sfq^{k}, \sfq^{k+1}-\sfq^{k} \rangle 
    +\langle \nabla_\sfq f_T(\sfq^{k}, w^{k+1})-\nabla_\sfq f_T(\sfq^{k}, w^{k}), \sfq^{k+1}-\sfq^{k} \rangle + \frac{\mu_\sfq}{2}\norm{\sfq^{k+1}-\sfq^{k}}_1^2\\
    & \leq f_T(\sfq^{k}, w^{k+1})+\langle \nabla_\sfq^{k}, \sfq^{k+1}-\sfq^{k} \rangle +\gamma_{\sfq, w}\norm{w^{k+1}-w^{k}}_2\norm{\sfq^{k+1}-\sfq^{k}}_1
    + \frac{\mu_\sfq}{2}\norm{\sfq^{k+1}-\sfq^{k}}_1^2\\
    & \leq f_T(\sfq^{k}, w^{k})+\langle \nabla_\sfq^{k}, \sfq^{k+1}-\sfq^{k} \rangle + \langle \nabla_w^{k}, w^{k+1}-w^k\rangle +\gamma_{\sfq, w}\norm{w^{k+1}-w^{k}}_2\norm{\sfq^{k+1}-\sfq^{k}}_1
    + \frac{\mu_\sfq}{2}\norm{\sfq^{k+1}-\sfq^{k}}_1^2 \\
    &~~~~+\frac{\mu_w}{2}\norm{w^{k+1}-w^{k}}_2^2\\
    & \leq f_T(\sfq^{k}, w^{k})+\eta\langle \nabla_\sfq^{k}, v_\sfq^{k}-\sfq^{k} \rangle + \eta\langle \nabla_w^{k}, u_w^k-w^k\rangle +\gamma_{\sfq, w}\eta^2D_\sfq D_w
    + \frac{\eta^2\mu_\sfq D^2_\sfq}{2}
     + \frac{\eta^2\mu_w D^2_w}{2}\\
     & \leq f_T(\sfq^{k}, w^{k})+\eta\langle \nabla_\sfq^{k}, v_\sfq^{k}-\sfq^{k} \rangle + \eta\langle \hnabla_w^{k}, u_w^k-w^k\rangle +\eta\langle \nabla_w^{k}- \hnabla_w^{k}, u_w^k-w^k\rangle +\gamma_{\sfq, w}\eta^2D_\sfq D_w\\ 
     &~~~~+ \frac{\eta^2\mu_\sfq D^2_\sfq}{2}
     + \frac{\eta^2\mu_w D^2_w}{2}
\end{align*}
Define $v_\ast^k\triangleq\argmin\limits_{v\in\cV}\langle \nabla_\sfq^k, v\rangle$ and $\alpha^k\triangleq \langle \nabla_\sfq^k, v_\sfq^k-v_\ast^k\rangle$.  Also, define $u_\ast^k\triangleq\argmin\limits_{u\in\cW}\langle \nabla_w^k, u\rangle$. Hence, noting that $\langle \hnabla_w^{k}, u_w^k-w^k\rangle \leq \langle \hnabla_w^{k}, u_\ast^k-w^k\rangle$ (which follows from the definition of $u_w^k$ in Step~\ref{step:u_w_k} in Algorithm~\ref{Alg:stFW}), the bound on $f_T(\sfq^{k+1}, w^{k+1})$ above can be further upper bounded as
\begin{align*}
&f_T(\sfq^{k}, w^{k})+\eta\langle \nabla_\sfq^{k}, v_\sfq^{k}-\sfq^{k} \rangle + \eta\bracket*{\langle \hnabla_w^{k}, u_\ast^k-w^k\rangle +\langle \nabla_w^{k}- \hnabla_w^{k}, u_w^k-w^k \rangle} +\gamma_{\sfq, w}\eta^2D_\sfq D_w\\
&~~~+ \frac{\eta^2\mu_\sfq D^2_\sfq}{2}
     + \frac{\eta^2\mu_w D^2_w}{2}\\
\leq &f_T(\sfq^{k}, w^{k})+\eta\langle \nabla_\sfq^{k}, v_\sfq^{k}-\sfq^{k} \rangle + \eta\bracket*{\langle \nabla_w^{k}, u_\ast^k-w^k\rangle +\langle \hnabla_w^{k}- \nabla_w^{k}, u_\ast^k-w^k \rangle+\langle \nabla_w^{k}- \hnabla_w^{k}, u_w^k-w^k \rangle}\\
&+\gamma_{\sfq, w}\eta^2D_\sfq D_w+ \frac{\eta^2\mu_\sfq D^2_\sfq}{2}
     + \frac{\eta^2\mu_w D^2_w}{2}\\
\leq &f_T(\sfq^{k}, w^{k})+\eta\langle \nabla_\sfq^{k}, v_\sfq^{k}-\sfq^{k} \rangle + +\eta \bracket*{\langle \nabla_w^{k}, u_\ast^k-w^k\rangle +\langle \hnabla_w^{k}- \nabla_w^{k}, u_\ast^k-u_w^k \rangle}+\gamma_{\sfq, w}\eta^2D_\sfq D_w\\
&~~~+ \frac{\eta^2\mu_\sfq D^2_\sfq}{2}
     + \frac{\eta^2\mu_w D^2_w}{2}\\
\leq & f_T(\sfq^{k}, w^{k})+\eta\bracket*{\langle \nabla_\sfq^{k}, v_\ast^{k}-\sfq^{k} \rangle + \alpha^k}+\eta\bracket*{\langle \nabla_w^{k}, u_\ast^k-w^k\rangle +\langle \hnabla_w^{k}- \nabla_w^{k}, u_\ast^k-u_w^k\rangle} +\gamma_{\sfq, w}\eta^2D_\sfq D_w
\\
&~~~+ \frac{\eta^2\mu_\sfq D^2_\sfq}{2}
     + \frac{\eta^2\mu_w D^2_w}{2}\\
\leq &f_T(\sfq^{k}, w^{k})+\eta\bracket*{\langle \nabla_\sfq^{k}, v_\ast^{k}-\sfq^{k} \rangle + \langle \nabla_w^{k}, u_\ast^k-w^k\rangle} + \eta\alpha^k+\eta D_w \norm{\hnabla_w^k - \nabla_w^k}_2 + \gamma_{\sfq, w}\eta^2D_\sfq D_w
\\ &~~~+ \frac{\eta^2\mu_\sfq D^2_\sfq}{2}
     + \frac{\eta^2\mu_w D^2_w}{2}
\end{align*}
Note $\langle \nabla_\sfq^{k}, v_\ast^{k}-\sfq^{k} \rangle +\langle \nabla_w^{k}, u_\ast^k-w^k\rangle = -\gap_{f_T}(\sfq^k, w^k)$. Moreover, with standard bounds on then tail of Laplacian and Gaussian random variables, with probability at least $1 - \beta$, for all $k \in [K]$, $\alpha^k \leq \sigma_\sfq \log(2JK/\beta)$ and $\norm{\hnabla_w^k - \nabla_w^k}_2 \leq \sigma_w\sqrt{d\log(2K/\beta)}$. We will condition on this event for the rest of the proof. Hence, the bound becomes: 
\begin{align*}
&f_T(\sfq^{k}, w^{k})-\eta\gap_{f_T}(\sfq^k, w^k)
+ \eta \sigma_\sfq \log\paren*{\frac{2JK}{\beta}} 
+ \eta D_w\sigma_w\sqrt{d\log\paren*{\frac{2K}{\beta}}}
+ \gamma_{\sfq, w}\eta^2D_\sfq D_w\\
&~ + \frac{\eta^2\mu_\sfq D^2_\sfq}{2}
+ \frac{\eta^2\mu_w D^2_w}{2} 
\end{align*}
Rearranging terms, and then averaging over $k\in [K]$, we get 
\begin{align}
    \frac{1}{K}\sum_{k = 1}^K\gap_{f_T}(\sfq^k, w^k)
    & \leq \frac{f_T(\sfq^0, w^0)-f_T(\sfq^{K+1}, w^{K+1})}{\eta K} + \eta \bracket*{\gamma_{\sfq, w}D_\sfq D_w
    +\frac{\mu_\sfq D^2_\sfq}{2}
     + \frac{\mu_w D^2_w}{2} } \\
    & \quad +\sigma_\sfq \log\paren*{\frac{2JK}{\beta}}
    + D_w\sigma_w\sqrt{d\log\paren*{\frac{2K}{\beta}}} \nonumber\\
    & \leq \frac{D_\sfq \gamma_\sfq+D_w\gamma_w}{\eta K}
    +\eta \bracket*{\gamma_{\sfq, w}D_\sfq D_w
    +\frac{\mu_\sfq D^2_\sfq+\mu_w D^2_w}{2}
     }\\
    & \quad + \sqrt{K}\bracket*{\sigma^0_\sfq \log\paren*{\frac{2JK}{\beta}}
    + D_w\sigma^0_w\sqrt{d\log\paren*{\frac{2K}{\beta}}}}.  \label{ineq:A-by-2}
\end{align}
Optimizing this bound in $\eta$ and $K$ results in the settings of $K$ and $\eta$ in the theorem statement. Substituting with these settings and  simplifying, we get that the average gap is upper bounded by $\mathsf{A}/2$, where $\mathsf{A}$ is the bound in the theorem; namely, $\mathsf{A}= 5\sqrt{\Bar{D}\paren*{\sigma^0_\sfq\log\paren*{\tfrac{\Bar{D}J}{\sigma^0_\sfq \beta}}+D_w\sigma^0_w\sqrt{d\log\paren*{\tfrac{\Bar{D}}{D_w\sigma^0_w \beta}}}}}.$

Now, to conclude the proof, we show that $\gap_{f_T}(\h{\sfq}, \h{w})\leq \frac{1}{K}\sum_{k=1}^K\gap_{f_T}(\sfq^k, w^k)+\mathsf{A}/2$. By the definition of $\h{\sfq}, \h{w}$ and using a similar analysis as above (and using the tail bounds on the Gaussian and Laplace r.v.s as before), observe that $\gap_{f_T}(\h{\sfq}, \h{w})$ can be upper bounded as
\begin{align*}
  \gap_{f_T}(\h{\sfq}, \h{w})
  & = \gap_{f_T}(\sfq^{k^\ast}, w^{k^\ast})\\
    &\leq \min\limits_{k\in [K]}\gap_{f_T}(\sfq^k, w^k)+ \sigma_\sfq\log\paren*{\frac{2JK}{\beta}}
    + D_w\sigma_w\sqrt{d\log\paren*{\frac{2K}{\beta}}}\\
    & \leq \frac{1}{K} \sum_{k = 1}^K\gap_{f_T}(\sfq^k, w^k) + \sqrt{K}\bracket*{ \sigma^0_\sfq\log\paren*{\frac{2JK}{\beta}} + D_w\sigma^0_w\sqrt{d\log\paren*{\frac{2K}{\beta}}}}.
\end{align*}
Observe that the term $\sqrt{K}\bracket*{ \sigma^0_\sfq\log(\frac{2JK}{\beta})+D_w\sigma^0_w\sqrt{d\log(\frac{2K}{\beta})}}$ above is the last term in (\ref{ineq:A-by-2}). Hence, by substituting with the values of $K$ and $\eta$, we can show that this term is upper bounded by $\mathsf{A}/2$. This leads to the following: 
\begin{align*}
    \gap_{f_T}(\h{\sfq}, \h{w})
    & \leq \frac{1}{K}\sum_{k = 1}^K\gap_{f_T}(\sfq^k, w^k) + \mathsf{A}/2,
\end{align*}
which completes the proof.
\end{proof}

\textbf{Proof of Corollary~\ref{th:stFWadap}}. 

\stFWadap*

More precisely, the stationarity gap is bounded as 
$$\gap_{L_T}(\h{\sfq}, \h{w})\leq \frac{32(1+2\mu
  \hat{r}^2)^{\frac{1}{4}}(\Lambda\hat{r})^{\frac{3}{2}}r\sqrt{(\Lambda
    \hat{r}+Y)\mu\log(\frac{mn}{\beta})}\log^{\frac{1}{4}}(1/\delta)}{\sqrt{\varepsilon
    n}}.$$  
This is done by choosing $K$ and $\eta$ as follows:
\begin{align*}
  K = \frac{\varepsilon n (\Lambda \hat{r}+Y)\sqrt{1+2\mu\hat{r}^2}}{4\Lambda\hat{r}r^2\mu\log\paren*{\frac{mn}{\beta}}\sqrt{\log(\frac{1}{\delta})}}
  \qquad
  \eta = \frac{\sqrt{2}(\Lambda \hat{r}+Y)}{\Lambda\hat{r}\sqrt{(1+2\mu\hat{r}^2)K}}.
\end{align*}

Hence, bound (\ref{eq:disbound}) implies that w.p. $\geq
1-2\beta$ over the choice of the public and private datasets and the
algorithm's internal randomness, the expected loss of the predictor
$h_{\h{w}}$ (defined by the output $\h{w}$) w.r.t. the target domain
is bounded as
\begin{align*}
\cL(\sP, h_{\h{w}})
\leq& L_T(\h{\sfq}, \h{w})+\frac{2\log(m+n)}{\mu}
+  \frac{2\Lambda r(\Lambda r + Y)^2}{\sqrt{n}}
+ (\Lambda r+Y)^2 \sqrt{\frac{\log \frac{1}{\beta}}{2n}}+\eta_\sH(S, \wt T).
\end{align*}


\subsection{Instantiating Algorithm~\ref{Alg:stFW} with \texorpdfstring{$L_T(\sfq, w)$}{LT(q, w)}}\label{app:stFW-adap} 

We now show that the smooth approximation of our adaptation objective function $L_T(\sfq, w)=\sum_{i=1}^m \sfq_i\ell(h_w(x_i), y_i)+4\Lambda^2 \wt{F}_T(\sfq)$ (where $\ell(h_w(x_i), y_i)=(\langle w, x_i \rangle - y_i)^2$ is the squared loss) is an instance of the family of functions $f_T$ in Theorem~\ref{th:stFW}. 

Recall that we assume that $\cW\subseteq \mathbb{B}_2^d(\Lambda)$ (the $\norm{\cdot}_2$ ball of radius $\Lambda$ in $\Rset^d$), $\cX\subset \mathbb{B}_2^d(r)$, and $\cY\subseteq [-Y, Y]$ for some $Y>0$. Also,  recall that we denote the maximum norm of the feature vectors in the public dataset, $\max\limits_{i\in [m]}\norm{x_i}_2$, by $\hat{r}$. 

First, note that $\cQ$ in Algorithm~\ref{Alg:stFW} is instantiated with the simplex $\Delta_m$ and hence, $\cV$ is $\{\be_1, \ldots, \be_m\}$.

Second, since the private dataset $T$ only appears in $\wt F_T$, note that $\nabla_w L_T(\sfq, w)$ does not involve $T$. Thus, $\sigma_w$ in Algorithm~\ref{Alg:stFW} can be set to zero. That is, we do not need to privatize the Frank-Wolfe steps over $w$. 

Third, note that the global $\norm{\cdot}_\infty$-sensitivity of $\nabla_\sfq L_T(\sfq, w), \tau_\sfq$, is the same as that of $4\Lambda^2\nabla_\sfq \wt F$, which follows from Corollary~\ref{cor:wtf} (Part~3), namely, $\tau_\sfq=\frac{8\Lambda^2\mu r^2\hat{r}^2}{n}$, where $\mu$ is the approximation parameter of the soft-max and $\hat{r}=\max\limits_{i\in [m]}\norm{x_i}_2$. 

Fourth, note that $L_T(\cdot, \cdot)$ is $\paren[\Big]{(\gamma_\sfq, \norm{\cdot}_1), (\gamma_w, \norm{\cdot}_2)}$-Lipschitz, where $\gamma_\sfq=(\Lambda \hat{r}+Y)^2+4\Lambda^2\hat{r}^2$, which follows from $\norm{\nabla_\sfq L_T(\sfq, w)}_\infty \leq \max\limits_{i\in [m]}\ell(h_w(x_i), y_i)+4\Lambda^2\norm{\nabla_\sfq \wt{F}(\sfq)}_\infty$ together with Corollary~\ref{cor:wtf} (Part~4), and $\gamma_w=2(\Lambda \hat{r}+Y)\hat{r}$, which follows directly from the $\norm{\cdot}_2$-bound on the gradient of the squared loss over $\mathbb{B}_2^d(\Lambda)$. 

Moreover, $L_T(\cdot, \cdot)$ is $\paren[\Big]{(\mu_\sfq, \norm{\cdot}_1), (\mu_w, \norm{\cdot}_2)}$-smooth, where $\mu_\sfq = 4\Lambda^2\mu \hat{r}^4$, which follows from the fact that the smoothness of $L_T$ w.r.t. $\sfq$ is given by the smoothness of $4\Lambda^2 \wt{F}_T$, which follows from Corollary~\ref{cor:wtf} (Part~2), and $\mu_w = \hat{r}^2$, which follows from the fact that the squared loss $\ell(h_w(x), y)$ is $\norm{x}_2^2$-smooth w.r.t. $\norm{\cdot}_2$.

Additionally, the condition on $\norm{\nabla_\sfq L_T(\sfq, w)-\nabla_\sfq L_T(\sfq, w')}_\infty$ in Theorem~\ref{th:stFW} is satisfied in our case with $\gamma_{\sfq, w}=2\hat{r}(\Lambda \hat{r}+Y)$, which follows from the fact that $\norm{\nabla_\sfq L_T(\sfq, w)-\nabla_\sfq L_T(\sfq, w')}_\infty=\max_{i\in [m]}\lvert \ell(h_w(x_i), y_i) - \ell(h_{w'}(x_i), y_i)\rvert$ together with the Lipschitzness property of the squared loss over $\mathbb{B}_2^d(\Lambda)$. 

Finally, note that $D_\sfq=2$ and $D_w=2\Lambda$. 



\end{document}